\documentclass[10pt,letterpaper,english,latin9,utf8]{article}
\usepackage[T1]{fontenc}
\usepackage[latin9]{inputenc}
\usepackage{color}
\usepackage{url}
\usepackage{amsmath}
\usepackage{amsthm}
\usepackage{amssymb}
\PassOptionsToPackage{normalem}{ulem}
\usepackage{ulem}

\makeatletter

\pdfpageheight\paperheight
\pdfpagewidth\paperwidth

\theoremstyle{plain}
\newtheorem{thm}{\protect\theoremname}
  \theoremstyle{plain}
  \newtheorem{prop}[thm]{\protect\propositionname}
  \theoremstyle{definition}
  \newtheorem{defn}[thm]{\protect\definitionname}
  \theoremstyle{plain}
  \newtheorem{lem}[thm]{\protect\lemmaname}

\usepackage[final,nonatbib]{nips_2017}
\PassOptionsToPackage{square,numbers}{natbib}

\usepackage[dvipsnames]{xcolor}
\usepackage{natbib}

\usepackage[T1]{fontenc}

\newcommand\myshade{55}
\colorlet{mylinkcolor}{violet}
\colorlet{mycitecolor}{blue}
\colorlet{myurlcolor}{BrickRed}
\usepackage{hyperref}       
\hypersetup{
  linkcolor  = mylinkcolor!\myshade!black,
 citecolor  = mycitecolor!\myshade!black,
  urlcolor   = myurlcolor!\myshade!black,
  colorlinks = true,
}

\usepackage{url}            
\usepackage{booktabs}       
\usepackage{amsfonts}       
\usepackage{nicefrac}       
\usepackage{microtype}      
\usepackage[resetlabels,labeled]{multibib}
\newcites{sup}{References}

\usepackage{algorithmic}
\usepackage{algorithm}
\usepackage{amsmath}
\usepackage{amsthm}
\usepackage[mathscr]{eucal}

\usepackage{mathtools}

\usepackage{soul}
\usepackage{subfig}

\usepackage{url}
\usepackage{tikz}
\usepackage{thm-restate}
\usepackage{thmtools}
\usepackage{wrapfig}

\newcommand{\diag}{\mathop{\mathrm{diag}}}

\DeclareMathOperator*{\plim}{\mathrm{plim}}

\allowdisplaybreaks

\newcommand{\lyxdot}{.}

\makeatother

\usepackage{babel}
  \providecommand{\definitionname}{Definition}
  \providecommand{\lemmaname}{Lemma}
  \providecommand{\propositionname}{Proposition}
\providecommand{\theoremname}{Theorem}

\begin{document}
\newcommand{\ourtitle}{A Linear-Time Kernel Goodness-of-Fit Test}
\theoremstyle{remark} 
\newtheorem{remark}{Remark} 
\date{}

\title{\ourtitle{}}

\author{Wittawat Jitkrittum \\
\vphantom{{\'E}}Gatsby Unit, UCL\\
{\small \url{wittawatj@gmail.com}} \\  \And 
Wenkai Xu \\
\vphantom{{\'E}}Gatsby Unit, UCL\\
{\small \url{wenkaix@gatsby.ucl.ac.uk}}\\ \And
Zolt{\'a}n Szab{\'o}\thanks{Zolt{\'a}n Szab{\'o}'s ORCID ID: 0000-0001-6183-7603. Arthur Gretton's ORCID ID: 0000-0003-3169-7624.} \\
CMAP, {\'E}cole Polytechnique \\
{\small \url{zoltan.szabo@polytechnique.edu}} \\ \And
Kenji Fukumizu \\
The Institute of Statistical Mathematics\\
\url{fukumizu@ism.ac.jp} \\ \And
Arthur Gretton$^*$ \\
Gatsby Unit, UCL \\
\url{arthur.gretton@gmail.com}}

\maketitle
\vspace{-2mm}
\begin{abstract}
\vspace{-2mm}
We propose a novel adaptive test of goodness-of-fit, with computational cost linear in the number of samples. We learn the test features that best indicate the differences between observed samples and a reference model, by minimizing the false negative rate. These features are constructed via Stein's method, meaning that it is not necessary to compute the normalising constant of the model. We analyse the asymptotic Bahadur efficiency of the new test, and prove that under a mean-shift alternative, our test always has greater relative efficiency than a previous linear-time kernel test, regardless of the choice of parameters for that test. In experiments, the performance of our method exceeds that of the earlier linear-time test, and matches or exceeds the power of a quadratic-time kernel test. In high dimensions and where model structure may be exploited, our goodness of fit test performs far better than a quadratic-time two-sample test based on the Maximum Mean Discrepancy, with samples drawn from the model.
\vspace{-3mm}
\end{abstract}

\section{Introduction}

\vspace{-1mm}The goal of goodness of fit testing is to determine
how well a model density $p(\mathbf{x})$ fits an observed sample
$\mathsf{D}=\{\mathbf{x}_{i}\}_{i=1}^{n}\subset\mathcal{X}\subseteq\mathbb{R}^{d}$
from an unknown distribution $q(\mathbf{x})$. This goal may be achieved
via a hypothesis test, where the null hypothesis $H_{0}\colon p=q$
is tested against $H_{1}\colon p\neq q$. The problem of testing goodness
of fit has a long history in statistics \cite{Massey51}, with a number
of tests proposed for particular parametric models. Such tests can
require space partitioning \cite{GyoMeu90,BeiGyoLug94}, which works
poorly in high dimensions; or closed-form integrals under the model,
which may be difficult to obtain, besides in certain special cases
\cite{BaringhausHenze88,BowFos93,SzeRiz05,Rizzo09}. An alternative
is to conduct a two-sample test using samples drawn from \emph{both}
$p$ and $q$. This approach was taken by \cite{LloGha15}, using
a test based on the (quadratic-time) Maximum Mean Discrepancy \cite{Gretton2012},
however this does not take advantage of the known structure of $p$
(quite apart from the increased computational cost of dealing with
samples from $p$).

More recently, measures of discrepancy with respect to a model have
been proposed based on Stein's method \cite{leyReiSwa2017}. A Stein
operator for $p$ may be applied to a class of test functions, yielding
functions that have zero expectation under $p$. Classes of test functions
can include the $W^{2,\infty}$ Sobolev space \cite{GorMac2015},
and reproducing kernel Hilbert spaces (RKHS) \cite{OatGirCho2016}.
Statistical tests have been proposed by \cite{Chwialkowski2016,LiuLeeJor2016}
based on classes of Stein transformed RKHS functions, where the test
statistic is the norm of the smoothness-constrained function with
largest expectation under $q$ . We will refer to this statistic as
the Kernel Stein Discrepancy (KSD). For consistent tests, it is sufficient
to use $C_{0}$-universal kernels \cite[Definition 4.1]{Carmeli2010},
as shown by \cite[Theorem 2.2]{Chwialkowski2016}, although inverse
multiquadric kernels may be preferred if uniform tightness is required
\cite{GorMac2017}.\footnote{Briefly, \cite{GorMac2017} show that when an exponentiated quadratic
kernel is used, a sequence of sets $\mathsf{D}$ may be constructed
that does not correspond to any $q$, but for which the KSD nonetheless
approaches zero. In a statistical testing setting, however, we assume
identically distributed samples from $q$, and the issue does not
arise.}

The minimum variance unbiased estimate of the KSD is a U-statistic,
with computational cost quadratic in the number $n$ of samples from
$q$. It is desirable to reduce the cost of testing, however, so that
larger sample sizes may be addressed. A first approach is to replace
the U-statistic with a running average with linear cost, as proposed
by \cite{LiuLeeJor2016} for the KSD, but this results in an increase
in variance and corresponding decrease in test power. An alternative
approach is to construct explicit features of the distributions, whose
empirical expectations may be computed in linear time. In the two-sample
and independence settings, these features were initially chosen at
random by \cite{EppSin86,ChwRamSejGre15,ZhaFilGreSej17}. More recently,
features have been constructed explicitly to maximize test power in
the two-sample \cite{Jitkrittum2016a} and independence testing \cite{Jitkrittum2016}
settings, resulting in tests that are not only more interpretable,
but which can yield performance matching quadratic-time tests.

We propose to construct explicit linear-time features for testing
goodness of fit, chosen so as to maximize test power. These features
further reveal where the model and data differ, in a readily interpretable
way. Our first theoretical contribution is a derivation of the null
and alternative distributions for tests based on such features, and
a corresponding power optimization criterion. Note that the goodness-of-fit
test requires somewhat different strategies to those employed for
two-sample and independence testing \cite{Jitkrittum2016a,Jitkrittum2016},
which become computationally prohibitive in high dimensions for the
Stein discrepancy (specifically, the normalization used in prior work
to simplify the asymptotics would incur a cost cubic in the dimension
$d$ and the number of features in the optimization). Details may
be found in Section \ref{sec:main_results}.

Our second theoretical contribution, given in Section \ref{sec:bahadur_efficiency},
is an analysis of the relative Bahadur efficiency of our test vs the
linear time test of \cite{LiuLeeJor2016}: this represents the relative
rate at which the p-value decreases under $H_{1}$ as we observe more
samples. We prove that our test has greater asymptotic Bahadur efficiency
relative to the test of \cite{LiuLeeJor2016}, for Gaussian distributions
under the mean-shift alternative. This is shown to hold regardless
of the bandwidth of the exponentiated quadratic kernel used for the
earlier test. The proof techniques developed are of independent interest,
and we anticipate that they may provide a foundation for the analysis
of relative efficiency of linear-time tests in the two-sample and
independence testing domains. In experiments (Section \ref{sec:experiments}),
our new linear-time test is able to detect subtle local differences
between the density $p(\mathbf{x})$, and the unknown $q(\mathbf{x})$
as observed through samples. We show that our linear-time test constructed
based on optimized features has comparable performance to the quadratic-time
test of \cite{Chwialkowski2016,LiuLeeJor2016}, while uniquely providing
an explicit visual indication of where the model fails to fit the
data.\vspace{-2mm}

\section{Kernel Stein Discrepancy (KSD) Test \label{sec:kstein_test}}

\vspace{-1mm}We begin by introducing the Kernel Stein Discrepancy
(KSD) and associated statistical test, as proposed independently by
\cite{Chwialkowski2016} and \cite{LiuLeeJor2016}. Assume that the
data domain is a connected open set $\mathcal{X}\subseteq\mathbb{R}^{d}$.
Consider a Stein operator $T_{p}$ that takes in a multivariate function
$\mathbf{f}(\mathbf{x})=(f_{1}(\mathbf{x}),\ldots,f_{d}(\mathbf{x}))^{\top}\in\mathbb{R}^{d}$
and constructs a function $\left(T_{p}\mathbf{f}\right)(\mathbf{x})\colon\mathbb{R}^{d}\to\mathbb{R}$.
The constructed function has the key property that for all $\mathbf{f}$
in an appropriate function class, $\mathbb{E}_{\mathbf{x}\sim q}\left[(T_{p}\mathbf{f})(\mathbf{x})\right]=0$
if and only if $q=p$. Thus, one can use this expectation as a statistic
for testing goodness of fit.

The function class $\mathcal{F}^{d}$ for the function\textbf{ $\mathbf{f}$
}is chosen to be a unit-norm ball in a reproducing kernel Hilbert
space (RKHS) in \cite{Chwialkowski2016,LiuLeeJor2016}. More precisely,
let $\mathcal{F}$ be an RKHS associated with a positive definite
kernel $k\colon\mathcal{X}\times\mathcal{X}\to\mathbb{R}$. Let $\phi(\mathbf{x})=k(\mathbf{x},\cdot)$
denote a feature map of $k$ so that $k(\mathbf{x},\mathbf{x}')=\left\langle \phi(\mathbf{x}),\phi(\mathbf{x}')\right\rangle _{\mathcal{F}}$.
Assume that $f_{i}\in\mathcal{F}$ for all $i=1,\ldots,d$ so that
$\mathbf{f}\in\mathcal{F}\times\cdots\times\mathcal{F}:=\mathcal{F}^{d}$
where $\mathcal{F}^{d}$ is equipped with the standard inner product
$\left\langle \mathbf{f},\mathbf{g}\right\rangle _{\mathcal{F}^{d}}:=\sum_{i=1}^{d}\left\langle f_{i},g_{i}\right\rangle _{\mathcal{F}}$.
The kernelized Stein operator $T_{p}$ studied in \cite{Chwialkowski2016}
is $\left(T_{p}\mathbf{f}\right)(\mathbf{x}):=\sum_{i=1}^{d}\left(\frac{\partial\log p(\mathbf{x})}{\partial x_{i}}f_{i}(\mathbf{x})+\frac{\partial f_{i}(\mathbf{x})}{\partial x_{i}}\right)\stackrel{(a)}{=}\left\langle \mathbf{f},\boldsymbol{\xi}_{p}(\mathbf{x},\cdot)\right\rangle _{\mathcal{F}^{d}},$
where at $(a)$ we use the reproducing property of $\mathcal{F}$,
i.e., $f_{i}(\mathbf{x})=\left\langle f_{i},k(\mathbf{x},\cdot)\right\rangle _{\mathcal{F}}$,
and that $\frac{\partial k(\mathbf{x},\cdot)}{\partial x_{i}}\in\mathcal{F}$
\cite[Lemma 4.34]{Steinwart2008}, hence $\boldsymbol{\xi}_{p}(\mathbf{x},\cdot):=\frac{\partial\log p(\mathbf{x})}{\partial\mathbf{x}}k(\mathbf{x},\cdot)+\frac{\partial k(\mathbf{x},\cdot)}{\partial\mathbf{x}}$
is in $\mathcal{F}^{d}$. We note that the Stein operator presented
in \cite{LiuLeeJor2016} is defined such that $\left(T_{p}\mathbf{f}\right)(\mathbf{x})\in\mathbb{R}^{d}$.
This distinction is not crucial and leads to the same goodness-of-fit
test. Under appropriate conditions, e.g. that $\lim_{\|\mathbf{x}\|\to\infty}p(\mathbf{x})f_{i}(\mathbf{x})=0$
for all $i=1,\ldots,d$, it can be shown using integration by parts
that $\mathbb{E}_{\mathbf{x}\sim p}(T_{p}\mathbf{f})(\mathbf{x})=0$
for any $\mathbf{f}\in\mathcal{F}^{d}$ \cite[Lemma 5.1]{Chwialkowski2016}.
Based on the Stein operator, \cite{Chwialkowski2016,LiuLeeJor2016}
define the kernelized Stein discrepancy as 
\begin{equation}
S_{p}(q):=\sup_{\|\mathbf{f}\|_{\mathcal{F}^{d}}\le1}\mathbb{E}_{\mathbf{x}\sim q}\left\langle \mathbf{f},\boldsymbol{\xi}_{p}(\mathbf{x},\cdot)\right\rangle _{\mathcal{F}^{d}}\stackrel{(a)}{=}\sup_{\|\mathbf{f}\|_{\mathcal{F}^{d}}\le1}\left\langle \mathbf{f},\mathbb{E}_{\mathbf{x}\sim q}\boldsymbol{\xi}_{p}(\mathbf{x},\cdot)\right\rangle _{\mathcal{F}^{d}}=\|\mathbf{g}(\cdot)\|_{\mathcal{F}^{d}},\label{eq:ksd_measure}
\end{equation}
where at $(a)$, $\boldsymbol{\xi}_{p}(\mathbf{x},\cdot)$ is Bochner
integrable \cite[Definition A.5.20]{Steinwart2008} as long as $\mathbb{E}_{\mathbf{x}\sim q}\|\boldsymbol{\xi}_{p}(\mathbf{x},\cdot)\|_{\mathcal{F}^{d}}<\infty$,
and $\mathbf{g}(\mathbf{y}):=\mathbb{E}_{\mathbf{x}\sim q}\boldsymbol{\xi}_{p}(\mathbf{x},\mathbf{y})$
is what we refer to as the \emph{Stein witness function}. The Stein
witness function will play a crucial role in our new test statistic
in Section \ref{sec:main_results}. When a $C_{0}$-universal kernel
is used \cite[Definition 4.1]{Carmeli2010}, and as long as $\mathbb{E}_{\mathbf{x}\sim q}\|\nabla_{\mathbf{x}}\log p(\mathbf{x})-\nabla_{\mathbf{x}}\log q(\mathbf{x})\|^{2}<\infty,$
it can be shown that $S_{p}(q)=0$ if and only if $p=q$ \cite[Theorem 2.2]{Chwialkowski2016}. 

 The KSD $S_{p}(q)$ can be written as $S_{p}^{2}(q)=\mathbb{E}_{\mathbf{x}\sim q}\mathbb{E}_{\mathbf{x}'\sim q}h_{p}(\mathbf{x},\mathbf{x}'),$
where $h_{p}(\mathbf{x},\mathbf{y}):=\mathbf{s}_{p}^{\top}(\mathbf{x})\mathbf{s}_{p}(\mathbf{y})k(\mathbf{x},\mathbf{y})+\mathbf{s}_{p}^{\top}(\mathbf{y})\nabla_{\mathbf{x}}k(\mathbf{x},\mathbf{y})+\mathbf{s}_{p}^{\top}(\mathbf{x})\nabla_{\mathbf{y}}k(\mathbf{x},\mathbf{y})+\sum_{i=1}^{d}\frac{\partial^{2}k(\mathbf{x},\mathbf{y})}{\partial x_{i}\partial y_{i}},$
and $\mathbf{s}_{p}(\mathbf{x}):=\nabla_{\mathbf{x}}\log p(\mathbf{x})$
is a column vector. An unbiased empirical estimator of $S_{p}^{2}(q)$,
denoted by $\widehat{S^{2}}=\frac{2}{n(n-1)}\sum_{i<j}h_{p}(\mathbf{x}_{i},\mathbf{x}_{j})$
\cite[Eq.\ 14]{LiuLeeJor2016}, is a degenerate U-statistic under
$H_{0}$. For the goodness-of-fit test, the rejection threshold can
be computed by a bootstrap procedure. All these properties make $\widehat{S^{2}}$
a very flexible criterion to detect the discrepancy of $p$ and $q$:
in particular, it can be computed even if $p$ is known only up to
a normalization constant.  Further studies on nonparametric Stein
operators can be found in \cite{OatGirCho2016,GorMac2015}.

\textbf{Linear-Time Kernel Stein (LKS) Test} Computation of $\widehat{S^{2}}$
costs $\mathcal{O}(n^{2})$. To reduce this cost, a linear-time (i.e.,
$\mathcal{O}(n)$) estimator based on an incomplete U-statistic is
proposed in \cite[Eq.\ 17]{LiuLeeJor2016}, given by $\widehat{S_{l}^{2}}:=\frac{2}{n}\sum_{i=1}^{n/2}h_{p}(\mathbf{x}_{2i-1},\mathbf{x}_{2i}),$
where we assume $n$ is even for simplicity. Empirically \cite{LiuLeeJor2016}
observed that the linear-time estimator performs much worse (in terms
of test power) than the quadratic-time U-statistic estimator, agreeing
with our findings presented in Section \ref{sec:experiments}. 

\section{New Statistic: The Finite Set Stein Discrepancy (FSSD)}

\label{sec:main_results}Although shown to be powerful, the main drawback
of the KSD test is its high computational cost of $\mathcal{O}(n^{2})$.
The LKS test is one order of magnitude faster. Unfortunately, the
decrease in the test power outweighs the computational gain \cite{LiuLeeJor2016}.
We therefore seek a variant of the KSD statistic that can be computed
in linear time, and whose test power is comparable to the KSD test. 

\textbf{Key Idea} The fact that $S_{p}(q)=0$ if and only if $p=q$
implies that $\mathbf{g}(\mathbf{v})=\mathbf{0}$ for all $\mathbf{v}\in\mathcal{X}$
if and only if $p=q$, where $\mathbf{g}$ is the Stein witness function
in (\ref{eq:ksd_measure}). One can see \textbf{$\mathbf{g}$ }as
a function witnessing the differences of $p,q$, in such a way that
$|g_{i}(\mathbf{v})|$ is large when there is a discrepancy in the
region around $\mathbf{v}$, as indicated by the $i^{th}$ output
of $\mathbf{g}$. The test statistic of \cite{LiuLeeJor2016,Chwialkowski2016}
is essentially given by the degree of ``flatness'' of $\mathbf{g}$
as measured by the RKHS norm $\|\cdot\|_{\mathcal{F}^{d}}$. The core
of our proposal is to use a different measure of flatness of $\mathbf{g}$
which can be computed in linear time. 

The idea is to use a real analytic kernel $k$ which makes $g_{1},\ldots,g_{d}$
real analytic. If $g_{i}\neq0$ is an analytic function, then the
Lebesgue measure of the set of roots $\{\mathbf{x}\mid g_{i}(\mathbf{x})=0\}$
is zero \cite{Mityagin2015}. This property suggests that one can
evaluate $g_{i}$ at a finite set of locations $V=\{\mathbf{v}_{1},\ldots,\mathbf{v}_{J}\}$,
drawn from a distribution with a density (w.r.t. the Lebesgue measure).
If $g_{i}\neq0$, then almost surely $g_{i}(\mathbf{v}_{1}),\ldots,g_{i}(\mathbf{v}_{J})$
will not be zero. This idea was successfully exploited in recently
proposed linear-time tests of \cite{ChwRamSejGre15} and \cite{Jitkrittum2016a,Jitkrittum2016}.
Our new test statistic based on this idea is called the Finite Set
Stein Discrepancy (FSSD) and is given in Theorem \ref{thm:fssd}.
All proofs are given in the appendix. 

\begin{restatable}[The Finite Set Stein Discrepancy (FSSD)]{thm}{fssd}

\label{thm:fssd} Let $V=\{\mathbf{v}_{1},\ldots,\mathbf{v}_{J}\}\subset\mathbb{R}^{d}$
be random vectors drawn i.i.d. from a distribution $\eta$ which has
a density. Let $\mathcal{X}$ be a connected open set in $\mathbb{R}^{d}$.
Define $\mathrm{FSSD}_{p}^{2}(q):=\frac{1}{dJ}\sum_{i=1}^{d}\sum_{j=1}^{J}g_{i}^{2}(\mathbf{v}_{j})$.
Assume that \uline{1)} $k\colon\mathcal{X}\times\mathcal{X}\to\mathbb{R}$
is $C_{0}$-universal \cite[Definition 4.1]{Carmeli2010} and real
analytic i.e., for all $\mathbf{v}\in\mathcal{X}$, $f(\mathbf{x}):=k(\mathbf{x},\mathbf{v})$
is a real analytic function on $\mathcal{X}$. \uline{2)} $\mathbb{E}_{\mathbf{x}\sim q}\mathbb{E}_{\mathbf{x}'\sim q}h_{p}(\mathbf{x},\mathbf{x}')<\infty$.
\uline{3)} $\mathbb{E}_{\mathbf{x}\sim q}\|\nabla_{\mathbf{x}}\log p(\mathbf{x})-\nabla_{\mathbf{x}}\log q(\mathbf{x})\|^{2}<\infty$.
\uline{4)} $\lim_{\|\mathbf{x}\|\to\infty}p(\mathbf{x})\mathbf{g}(\mathbf{x})=0$.

Then, for any $J\ge1$, $\eta$-almost surely $\mathrm{FSSD}_{p}^{2}(q)=0$
if and only if $p=q$.

\end{restatable}

This measure depends on a set of $J$ test locations (or features)
$\{\mathbf{v}_{i}\}_{i=1}^{J}$ used to evaluate the Stein witness
function, where $J$ is fixed and is typically small. A kernel which
is  $C_{0}$-universal and real analytic is the Gaussian kernel $k(\mathbf{x},\mathbf{y})=\exp\left(-\frac{\|\mathbf{x}-\mathbf{y}\|_{2}^{2}}{2\sigma_{k}^{2}}\right)$
(see \cite[Proposition 3]{Jitkrittum2016} for the result on analyticity).
Throughout this work, we will assume all the conditions stated in
Theorem \ref{thm:fssd}, and consider only the Gaussian kernel. Besides
the requirement that the kernel be real and analytic, the remaining
conditions in Theorem \ref{thm:fssd} are the same as given in \cite[Theorem 2.2]{Chwialkowski2016}.
Note that if the FSSD is to be employed in a setting otherwise than
testing, for instance to obtain pseudo-samples converging to  $p$,
then stronger conditions may be needed \cite{GorMac2017}.

\subsection{Goodness-of-Fit Test with the FSSD Statistic}

Given a significance level $\alpha$ for the goodness-of-fit test,
the test can be constructed so that $H_{0}$ is rejected when $n\widehat{\mathrm{FSSD^{2}}}>T_{\alpha}$,
where $T_{\alpha}$ is the rejection threshold (critical value), and
$\widehat{\mathrm{FSSD^{2}}}$ is an empirical estimate of $\mathrm{FSSD}_{p}^{2}(q)$.
The threshold which guarantees that the type-I error (i.e., the probability
of rejecting $H_{0}$ when it is true) is bounded above by $\alpha$
is given by the $(1-\alpha$)-quantile of the null distribution i.e.,
the distribution of $n\widehat{\mathrm{FSSD^{2}}}$ under $H_{0}$.
In the following, we start by giving the expression for $\widehat{\mathrm{FSSD^{2}}}$,
and summarize its asymptotic distributions in Proposition \ref{prop:fssd_asymp_dists}.

Let $\boldsymbol{\Xi}(\mathbf{x})\in\mathbb{R}^{d\times J}$ such
that $[\boldsymbol{\Xi}(\mathbf{x})]_{i,j}=\xi_{p,i}(\mathbf{x},\mathbf{v}_{j})/\sqrt{dJ}$.
Define $\boldsymbol{\tau}(\mathbf{x}):=\mathrm{vec}(\boldsymbol{\Xi}(\mathbf{x}))\in\mathbb{R}^{dJ}$
where $\mathrm{vec}(\mathbf{M})$ concatenates columns of the matrix
$\mathbf{M}$ into a column vector. We note that $\boldsymbol{\tau}(\mathbf{x})$
depends on the test locations $V=\{\mathbf{v}_{j}\}_{j=1}^{J}$. Let
$\Delta(\mathbf{x},\mathbf{y}):=\boldsymbol{\tau}(\mathbf{x})^{\top}\boldsymbol{\tau}(\mathbf{y})=\mathrm{tr}(\mathbf{\Xi}(\mathbf{x})^{\top}\mathbf{\Xi}(\mathbf{y}))$.
Given an i.i.d. sample $\{\mathbf{x}_{i}\}_{i=1}^{n}\sim q$, a consistent,
unbiased estimator of $\mathrm{FSSD}_{p}^{2}(q)$ is {\small{}
\begin{align}
\widehat{\mathrm{FSSD^{2}}} & =\frac{1}{dJ}\sum_{l=1}^{d}\sum_{m=1}^{J}\frac{1}{n(n-1)}\sum_{i=1}^{n}\sum_{j\neq i}\xi_{p,l}(\mathbf{x}_{i},\mathbf{v}_{m})\xi_{p,l}(\mathbf{x}_{j},\mathbf{v}_{m})=\frac{2}{n(n-1)}\sum_{i<j}\Delta(\mathbf{x}_{i},\mathbf{x}_{j}),\label{eq:fssd_ustat}
\end{align}
}which is a one-sample second-order U-statistic with $\Delta$ as
its U-statistic kernel \cite[Section 5.1.1]{Serfling2009}. Being
a U-statistic, its asymptotic distribution can easily be derived.
We use $\stackrel{d}{\to}$ to denote convergence in distribution.
\begin{prop}[Asymptotic distributions of $\widehat{\mathrm{FSSD^{2}}}$]
\label{prop:fssd_asymp_dists} Let $Z_{1},\ldots,Z_{dJ}\stackrel{i.i.d.}{\sim}\mathcal{N}(0,1)$.
Let $\boldsymbol{\mu}:=\mathbb{E}_{\mathbf{x}\sim q}[\boldsymbol{\tau}(\mathbf{x})]$,
$\boldsymbol{\Sigma}_{r}:=\mathrm{cov}_{\mathbf{x}\sim r}[\boldsymbol{\tau}(\mathbf{x})]\in\mathbb{R}^{dJ\times dJ}$
for $r\in\{p,q\}$, and $\{\omega_{i}\}_{i=1}^{dJ}$ be the eigenvalues
of $\boldsymbol{\Sigma}_{p}=\mathbb{E}_{\mathbf{x}\sim p}[\boldsymbol{\tau}(\mathbf{x})\boldsymbol{\tau}^{\top}(\mathbf{x})]$.
Assume that $\mathbb{E}_{\mathbf{x}\sim q}\mathbb{E}_{\mathbf{y}\sim q}\Delta^{2}(\mathbf{x},\mathbf{y})<\infty$.
Then, for any realization of $V=\{\mathbf{v}_{j}\}_{j=1}^{J}$, the
following statements hold.
\begin{enumerate}
\item Under $H_{0}:p=q$, $n\widehat{\mathrm{FSSD^{2}}}\stackrel{d}{\to}\sum_{i=1}^{dJ}(Z_{i}^{2}-1)\omega_{i}$. 
\item Under $H_{1}:p\neq q$, if $\sigma_{H_{1}}^{2}:=4\boldsymbol{\mu}^{\top}\boldsymbol{\Sigma}_{q}\boldsymbol{\mu}>0$,
then $\sqrt{n}(\widehat{\mathrm{FSSD^{2}}}-\mathrm{FSSD^{2}})\stackrel{d}{\to}\mathcal{N}(0,\sigma_{H_{1}}^{2})$.
\end{enumerate}
\end{prop}
\begin{proof}
Recognizing that (\ref{eq:fssd_ustat}) is a degenerate U-statistic,
the results follow directly from \cite[Section 5.5.1, 5.5.2]{Serfling2009}.
\end{proof}
Claims 1 and 2 of Proposition \ref{prop:fssd_asymp_dists} imply that
under $H_{1}$, the test power (i.e., the probability of correctly
rejecting $H_{1}$) goes to 1 asymptotically, if the threshold $T_{\alpha}$
is defined as above. In practice, simulating from the asymptotic null
distribution in Claim 1 can be challenging, since the plug-in estimator
of $\boldsymbol{\Sigma}_{p}$ requires a sample from $p$, which is
not available. A straightforward solution is to draw sample from $p$,
either by assuming that $p$ can be sampled easily or by using a Markov
chain Monte Carlo (MCMC) method, although this adds an additional
computational burden to the test procedure. A more subtle issue is
that when dependent samples from $p$ are used in obtaining the test
threshold, the test may become more conservative than required for
i.i.d. data \cite{ChwSejGre14}. An alternative approach is to use
the plug-in estimate $\hat{\boldsymbol{\Sigma}}_{q}$ instead of $\boldsymbol{\Sigma}_{p}$.
The covariance matrix $\hat{\boldsymbol{\Sigma}}_{q}$ can be directly
computed from the data. This is the approach we take. Theorem \ref{thm:sigma_q_consistent}
guarantees that the replacement of the covariance in the computation
of the asymptotic null distribution still yields a consistent test.
We write $\mathbb{P}_{H_{1}}$ for the distribution of $n\widehat{\mathrm{FSSD^{2}}}$
under $H_{1}$.

\begin{restatable}[]{thm}{sigmaqconsistent}

\label{thm:sigma_q_consistent} Let $\hat{\boldsymbol{\Sigma}}_{q}:=\frac{1}{n}\sum_{i=1}^{n}\boldsymbol{\tau}(\mathbf{x}_{i})\boldsymbol{\tau}^{\top}(\mathbf{x}_{i})-[\frac{1}{n}\sum_{i=1}^{n}\boldsymbol{\tau}(\mathbf{x}_{i})][\frac{1}{n}\sum_{j=1}^{n}\boldsymbol{\tau}(\mathbf{x}_{j})]^{\top}$
with $\{\mathbf{x}_{i}\}_{i=1}^{n}\sim q$. Suppose that the test
threshold $T_{\alpha}$ is set to the $(1-\alpha)$-quantile of the
distribution of $\sum_{i=1}^{dJ}(Z_{i}^{2}-1)\hat{\nu_{i}}$ where
$\{Z_{i}\}_{i=1}^{dJ}\stackrel{i.i.d.}{\sim}\mathcal{N}(0,1)$, and
$\hat{\nu}_{1},\ldots,\hat{\nu}_{dJ}$ are eigenvalues of $\hat{\boldsymbol{\Sigma}}_{q}$.
Then, under $H_{0}$, asymptotically the false positive rate is $\alpha$.
Under $H_{1}$, for $\{\mathbf{v}_{j}\}_{j=1}^{J}$ drawn from a distribution
with a density, the test power $\mathbb{P}_{H_{1}}(n\widehat{\mathrm{FSSD^{2}}}>T_{\alpha})\to1$
as $n\to\infty$.

\end{restatable}

\begin{remark}

The proof of Theorem \ref{thm:sigma_q_consistent} relies on two facts.
First, under $H_{0}$, $\hat{\boldsymbol{\Sigma}}_{q}=\hat{\boldsymbol{\Sigma}}_{p}$
i.e., the plug-in estimate of $\boldsymbol{\Sigma}_{p}$. Thus, under
$H_{0}$, the null distribution approximated with $\hat{\boldsymbol{\Sigma}}_{q}$
is asymptotically correct, following the convergence of $\hat{\boldsymbol{\Sigma}}_{p}$
to $\boldsymbol{\Sigma}_{p}$. Second, the rejection threshold obtained
from the approximated null distribution is asymptotically constant.
Hence, under $H_{1}$, claim 2 of Proposition \ref{prop:fssd_asymp_dists}
implies that $n\widehat{\mathrm{FSSD^{2}}}\stackrel{d}{\to}\infty$
as $n\to\infty$, and consequently $\mathbb{P}_{H_{1}}(n\widehat{\mathrm{FSSD^{2}}}>T_{\alpha})\to1$.

\end{remark}

\subsection{Optimizing the Test Parameters}

\label{subsec:param_optimization}Theorem \ref{thm:fssd} guarantees
that the population quantity $\mathrm{FSSD^{2}}=0$ if and only if
$p=q$ for any choice of $\{\mathbf{v}_{i}\}_{i=1}^{J}$ drawn from
a distribution with a density. In practice, we are forced to rely
on the empirical $\widehat{\mathrm{FSSD^{2}}}$, and some test locations
will give a higher detection rate (i.e., test power) than others for
finite $n$. Following the approaches of \cite{Gretton2012a,Jitkrittum2016,Jitkrittum2016a,Sutherland2016},
we choose the test locations $V=\left\{ \mathbf{v}_{j}\right\} _{j=1}^{J}$
and kernel bandwidth $\sigma_{k}^{2}$ so as to maximize the test
power i.e., the probability of rejecting $H_{0}$ when it is false.
We first give an approximate expression for the test power when $n$
is large. 
\begin{prop}[Approximate test power of $n\widehat{\mathrm{FSSD^{2}}}$]
\label{prop:fssd_power} Under $H_{1}$, for large $n$ and fixed
$r$, the test power $\mathbb{P}_{H_{1}}(n\widehat{\mathrm{FSSD^{2}}}>r)\approx1-\Phi\left(\frac{r}{\sqrt{n}\sigma_{H_{1}}}-\sqrt{n}\frac{\mathrm{FSSD^{2}}}{\sigma_{H_{1}}}\right)$,
where $\Phi$ denotes the cumulative distribution function of the
standard normal distribution, and $\sigma_{H_{1}}$ is defined in
Proposition \ref{prop:fssd_asymp_dists}.
\end{prop}
\begin{proof}
$\mathbb{P}_{H_{1}}(n\widehat{\mathrm{FSSD^{2}}}>r)=\mathbb{P}_{H_{1}}(\widehat{\mathrm{FSSD^{2}}}>r/n)=\mathbb{P}_{H_{1}}\left(\sqrt{n}\frac{\widehat{\mathrm{FSSD^{2}}}-\mathrm{FSSD^{2}}}{\sigma_{H_{1}}}>\sqrt{n}\frac{r/n-\mathrm{FSSD^{2}}}{\sigma_{H_{1}}}\right)$.
For sufficiently large $n$, the alternative distribution is approximately
normal as given in Proposition \ref{prop:fssd_asymp_dists}. It follows
that $\mathbb{P}_{H_{1}}(n\widehat{\mathrm{FSSD^{2}}}>r)\approx1-\Phi\left(\frac{r}{\sqrt{n}\sigma_{H_{1}}}-\sqrt{n}\frac{\mathrm{FSSD^{2}}}{\sigma_{H_{1}}}\right)$. 
\end{proof}
Let $\boldsymbol{\zeta}:=\{V,\sigma_{k}^{2}\}$ be the collection
of all tuning parameters. Assume that $n$ is sufficiently large.
Following the same argument as in \cite{Sutherland2016}, in $\frac{r}{\sqrt{n}\sigma_{H_{1}}}-\sqrt{n}\frac{\mathrm{FSSD^{2}}}{\sigma_{H_{1}}}$,
we observe that the first term $\frac{r}{\sqrt{n}\sigma_{H_{1}}}=\mathcal{O}(n^{-1/2})$
going to 0 as $n\to\infty$, while the second term $\sqrt{n}\frac{\mathrm{FSSD^{2}}}{\sigma_{H_{1}}}=\mathcal{O}(n^{1/2})$,
dominating the first for large $n$. Thus, the best parameters that
maximize the test power are given by $\boldsymbol{\zeta}^{*}=\arg\max_{\boldsymbol{\zeta}}\mathbb{P}_{H_{1}}(n\widehat{\mathrm{FSSD^{2}}}>T_{\alpha})\approx\arg\max_{\boldsymbol{\zeta}}\frac{\mathrm{FSSD^{2}}}{\sigma_{H_{1}}}.$
Since $\mathrm{FSSD^{2}}$ and $\sigma_{H_{1}}$ are unknown, we divide
the sample $\{\mathbf{x}_{i}\}_{i=1}^{n}$ into two disjoint training
and test sets, and use the training set to compute $\frac{\widehat{\mathrm{FSSD^{2}}}}{\hat{\sigma}_{H_{1}}+\gamma}$,
where a small regularization parameter $\gamma>0$ is added for numerical
stability. The goodness-of-fit test is performed on the test set
to avoid overfitting. The idea of splitting the data into training
and test sets to learn good features for hypothesis testing was successfully
used in \cite{Sutherland2016,Jitkrittum2016,Jitkrittum2016a,Gretton2012a}. 

To find a local maximum of $\frac{\widehat{\mathrm{FSSD^{2}}}}{\hat{\sigma}_{H_{1}}+\gamma}$,
we use gradient ascent for its simplicity. The initial points of $\{\mathbf{v}_{i}\}_{i=1}^{J}$
are set to random draws from a normal distribution fitted to the training
data, a heuristic we found to perform well in practice. The objective
is non-convex in general, reflecting many possible ways to capture
the differences of $p$ and $q$. The regularization parameter $\gamma$
is not tuned, and is fixed to a small constant. Assume that $\nabla_{\mathbf{x}}\log p(\mathbf{x})$
costs $\mathcal{O}(d^{2})$ to evaluate. Computing $\nabla_{\boldsymbol{\zeta}}\frac{\widehat{\mathrm{FSSD^{2}}}}{\hat{\sigma}_{H_{1}}+\gamma}$
costs $\mathcal{O}(d^{2}J^{2}n)$. The computational complexity of
$n\widehat{\mathrm{FSSD^{2}}}$ and $\hat{\sigma}_{H_{1}}^{2}$ is
$\mathcal{O}(d^{2}Jn)$. Thus, finding a local optimum via gradient
ascent is still linear-time, for a fixed maximum number of iterations.
Computing $\hat{\boldsymbol{\Sigma}}_{q}$ costs $\mathcal{O}(d^{2}J^{2}n)$,
and obtaining all the eigenvalues of $\hat{\boldsymbol{\Sigma}}_{q}$
costs $\mathcal{O}(d^{3}J^{3})$ (required only once). If the eigenvalues
decay to zero sufficiently rapidly, one can approximate the asymptotic
null distribution with only a few eigenvalues. The cost to obtain
the largest few eigenvalues alone can be much smaller. 

\begin{remark}

Let $\hat{\boldsymbol{\mu}}:=\frac{1}{n}\sum_{i=1}^{n}\boldsymbol{\tau}(\mathbf{x}_{i})$.
It is possible to normalize the FSSD statistic to get a new statistic
$\hat{\lambda}_{n}:=n\hat{\boldsymbol{\mu}}^{\top}(\hat{\boldsymbol{\Sigma}}_{q}+\gamma\mathbf{I})^{-1}\hat{\boldsymbol{\mu}}$
where $\gamma\ge0$ is a regularization parameter that goes to 0 as
$n\to\infty$. This was done in the case of the ME (mean embeddings)
statistic of \cite{ChwRamSejGre15,Jitkrittum2016a}. The asymptotic
null distribution of this statistic takes the convenient form of $\chi^{2}(dJ)$
(independent of $p$ and $q$), eliminating the need to obtain the
eigenvalues of $\hat{\boldsymbol{\Sigma}}_{q}$. It turns out that
the test power criterion for tuning the parameters in this case is
the statistic $\hat{\lambda}_{n}$ itself. However, the optimization
is computationally expensive as $(\hat{\boldsymbol{\Sigma}}_{q}+\gamma\mathbf{I})^{-1}$
(costing $\mathcal{O}(d^{3}J^{3})$) needs to be reevaluated in each
gradient ascent iteration. This is not needed in our proposed FSSD
statistic.

\end{remark}

\section{Relative Efficiency and Bahadur Slope}

\label{sec:bahadur_efficiency}Both the linear-time kernel Stein (LKS)
and FSSD tests have the same computational cost of $\mathcal{O}(d^{2}n)$,
and are consistent, achieving maximum power of 1 as $n\to\infty$
under $H_{1}$. It is thus of theoretical interest to understand which
test is more sensitive in detecting the differences of $p$ and $q$.
This can be quantified by the \emph{Bahadur slope} of the test \cite{Bah1960}.
Two given tests can then be compared by computing the \emph{Bahadur
efficiency} (Theorem \ref{thm:effgaussmean}) which is given by the
ratio of the slopes of the two tests. We note that the constructions
and techniques in this section may be of independent interest, and
can be generalised to other statistical testing settings. 

We start by introducing the concept of Bahadur slope for a general
test, following the presentation of \cite{Gle1964,Gle1966}. Consider
a hypothesis testing problem on a parameter $\theta$. The test proposes
a null hypothesis $H_{0}:\theta\in\Theta_{0}$ against the alternative
hypothesis $H_{1}:\theta\in\Theta\backslash\Theta_{0}$, where $\Theta,\Theta_{0}$
are arbitrary sets.  Let $T_{n}$ be a test statistic computed from
a sample of size $n$, such that large values of $T_{n}$ provide
an evidence to reject $H_{0}$. We use $\plim$ to denote convergence
in probability, and write $\mathbb{E}_{r}$ for $\mathbb{E}_{\mathbf{x}\sim r}\mathbb{E}_{\mathbf{x}'\sim r}$. 

\textbf{Approximate Bahadur Slope (ABS)} For $\theta_{0}\in\Theta_{0}$,
let the asymptotic null distribution of $T_{n}$ be $F(t)=\lim_{n\to\infty}P_{\theta_{0}}(T_{n}<t)$,
where we assume that the CDF ($F$) is continuous and common to all
$\theta_{0}\in\Theta_{0}$. The continuity of $F$ will be important
later when Theorem \ref{thm:gleser1964_slope} and \ref{thm:gleser1964_efficiency}
are used to compute the slopes of LKS and FSSD tests. Assume that
there exists a continuous strictly increasing function $\rho:(0,\infty)\to(0,\infty)$
such that $\lim_{n\to\infty}\rho(n)=\infty$, and that $-2\plim_{n\to\infty}\frac{\log(1-F(T_{n}))}{\rho(n)}=c(\theta)$
where $T_{n}\sim P_{\theta}$, for some function $c$ such that $0<c(\theta_{A})<\infty$
for $\theta_{A}\in\Theta\backslash\Theta_{0}$, and $c(\theta_{0})=0$
when $\theta_{0}\in\Theta_{0}$. The function $c(\theta)$ is known
as the \emph{approximate Bahadur slope} (ABS) of the sequence $T_{n}$.
The quantifier ``approximate'' comes from the use of the asymptotic
null distribution instead of the exact one \cite{Bah1960}. Intuitively
the slope $c(\theta_{A})$, for $\theta_{A}\in\Theta\backslash\Theta_{0}$,
is the rate of convergence of p-values (i.e., $1-F(T_{n})$) to 0,
as $n$ increases. The higher the slope, the faster the p-value vanishes,
and thus the lower the sample size required to reject $H_{0}$ under
$\theta_{A}$.

\textbf{Approximate Bahadur Efficiency} Given two sequences of test
statistics, $T_{n}^{(1)}$ and $T_{n}^{(2)}$ having the same $\rho(n)$
(see Theorem \ref{thm:gleser1964_efficiency}), the approximate Bahadur
efficiency of $T_{n}^{(1)}$ relative to $T_{n}^{(2)}$ is defined
as $E(\theta_{A}):=c^{(1)}(\theta_{A})/c^{(2)}(\theta_{A})$ for $\theta_{A}\in\Theta\backslash\Theta_{0}$.
If $E(\theta_{A})>1$, then $T_{n}^{(1)}$ is asymptotically more
efficient than $T_{n}^{(2)}$ in the sense of Bahadur, for the particular
problem specified by $\theta_{A}\in\Theta\backslash\Theta_{0}$. We
now give approximate Bahadur slopes for two sequences of linear time
test statistics: the proposed $n\widehat{\mathrm{FSSD^{2}}}$, and
the LKS test statistic $\sqrt{n}\widehat{S_{l}^{2}}$ discussed in
Section \ref{sec:kstein_test}. 

\begin{restatable}[]{thm}{fssdslope}

\label{thm:fssd_slope}The approximate Bahadur slope of $n\widehat{\mathrm{FSSD^{2}}}$
is $c^{(\mathrm{FSSD})}:=\mathrm{FSSD^{2}}/\omega_{1}$, where $\omega_{1}$
is the maximum eigenvalue of $\boldsymbol{\Sigma}_{p}:=\mathbb{E}_{\mathbf{x}\sim p}[\boldsymbol{\tau}(\mathbf{x})\boldsymbol{\tau}^{\top}(\mathbf{x})]$
and $\rho(n)=n$. 

\end{restatable}

\begin{restatable}[]{thm}{lksslope}

\label{thm:lks_slope} The approximate Bahadur slope of the linear-time
kernel Stein (LKS) test statistic $\sqrt{n}\widehat{S_{l}^{2}}$ is
$c^{(\mathrm{LKS})}=\frac{1}{2}\frac{\left[\mathbb{E}_{q}h_{p}(\mathbf{x},\mathbf{x}')\right]^{2}}{\mathbb{E}_{p}\left[h_{p}^{2}(\mathbf{x},\mathbf{x}')\right]}$,
where $h_{p}$ is the U-statistic kernel of the KSD statistic, and
$\rho(n)=n$.

\end{restatable}

To make these results concrete, we consider the setting where $p=\mathcal{N}(0,1)$
and $q=\mathcal{N}(\mu_{q},1)$. We assume that both tests use the
Gaussian kernel $k(x,y)=\exp\left(-(x-y)^{2}/2\sigma_{k}^{2}\right)$,
possibly with different bandwidths. We write $\sigma_{k}^{2}$ and
$\kappa^{2}$ for the FSSD and LKS bandwidths, respectively. Under
these assumptions, the slopes given in Theorem \ref{thm:fssd_slope}
and Theorem \ref{thm:lks_slope} can be derived explicitly. The full
expressions of the slopes are given in Proposition \ref{prop:fssd_slope_gauss}
and Proposition \ref{prop:lks_slope_gauss} (in the appendix). By
\cite{Gle1964,Gle1966} (recalled as Theorem \ref{thm:gleser1964_efficiency}
in the supplement), the approximate Bahadur efficiency can be computed
by taking the ratio of the two slopes. The efficiency is given in
Theorem \ref{thm:effgaussmean}.

\begin{restatable}[Efficiency in the Gaussian mean shift problem]{thm}{effgaussmean}

\label{thm:effgaussmean}Let $E_{1}(\mu_{q},v,\sigma_{k}^{2},\kappa^{2})$
be the approximate Bahadur efficiency of $n\widehat{\mathrm{FSSD^{2}}}$
relative to $\sqrt{n}\widehat{S_{l}^{2}}$ for the case where $p=\mathcal{N}(0,1),q=\mathcal{N}(\mu_{q},1),$
and $J=1$ (i.e., one test location $v$ for $n\widehat{\mathrm{FSSD^{2}}}$).
Fix $\sigma_{k}^{2}=1$ for $n\widehat{\mathrm{FSSD^{2}}}$. Then,
for any $\mu_{q}\neq0$, for some $v\in\mathbb{R}$, and for any $\kappa^{2}>0$,
we have $E_{1}(\mu_{q},v,\sigma_{k}^{2},\kappa^{2})>2$.

\end{restatable}

When $p=\mathcal{N}(0,1)$ and $q=\mathcal{N}(\mu_{q},1)$ for $\mu_{q}\neq0$,
Theorem \ref{thm:effgaussmean} guarantees that our FSSD test is asymptotically
at least twice as efficient as the LKS test in the Bahadur sense.
We note that the efficiency is conservative in the sense that $\sigma_{k}^{2}=1$
regardless of $\mu_{q}$. Choosing $\sigma_{k}^{2}$ dependent on
$\mu_{q}$ will likely improve the efficiency further.

\section{Experiments}

\label{sec:experiments}In this section, we demonstrate the performance
of the proposed test on a number of problems. The primary goal is
to understand the conditions under which the test can perform well.

\begin{wrapfigure}{r}{0.3\textwidth} 
\vspace{-8mm} 
\includegraphics[width=0.99\linewidth]{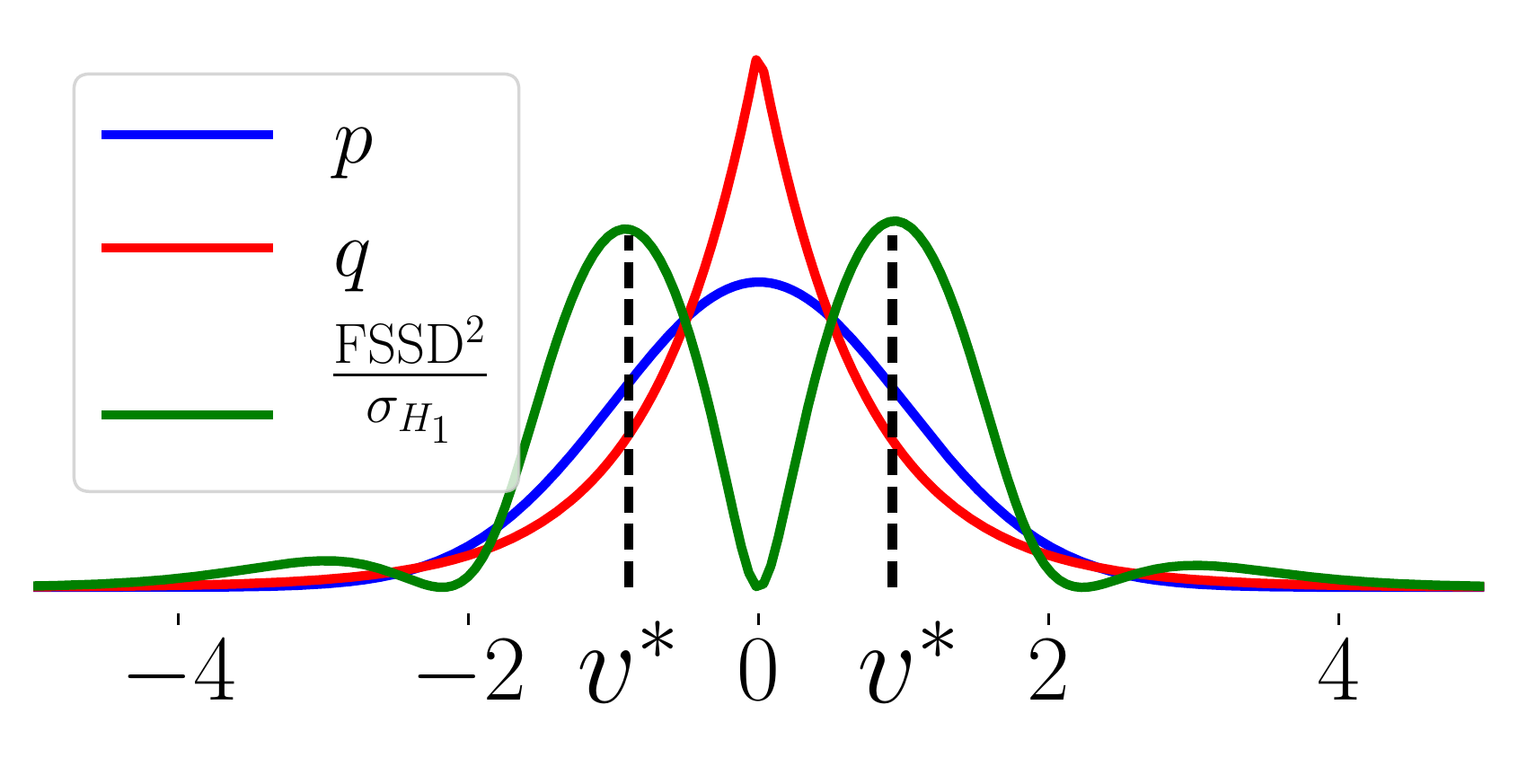}
\caption{The power criterion $\mathrm{FSSD^2}/\sigma_{H_1}$ as a function of test location $v$. 
}
\label{fig:obj_plot}
\vspace{-4mm}
\end{wrapfigure} 

\paragraph{Sensitivity to Local Differences }

We start by demonstrating that the test power objective $\mathrm{FSSD^{2}}/\sigma_{H_{1}}$
captures local differences of $p$ and $q$, and that interpretable
features $v$ are found. Consider a one-dimensional problem in which
$p=\mathcal{N}(0,1)$ and $q=\mathrm{Laplace}(0,1/\sqrt{2})$, a zero-mean
Laplace distribution with scale parameter $1/\sqrt{2}$. These parameters
are chosen so that $p$ and $q$ have the same mean and variance.
Figure \ref{fig:obj_plot} plots the (rescaled) objective as a function
of $v$. The objective illustrates that the best features (indicated
by $v^{*}$) are at the most discriminative locations.

\textbf{Test Power} We next investigate the power of different tests
on two problems:

1.\textbf{\,Gaussian vs. Laplace}: $p(\mathbf{x})=\mathcal{N}(\mathbf{x}|\mathbf{0},\mathbf{I}_{d})$
and $q(\mathbf{x})=\prod_{i=1}^{d}\mathrm{Laplace}(x_{i}|0,1/\sqrt{2})$
where the dimension $d$ will be varied. The two distributions have
the same mean and variance. The main characteristic of this problem
is local differences of $p$ and $q$ (see Figure \ref{fig:obj_plot}).
Set $n=1000$.

2.\textbf{\,Restricted Boltzmann Machine} (RBM): $p(\mathbf{x})$
is the marginal distribution of $p(\mathbf{x},\mathbf{h})=\frac{1}{Z}\exp\left(\mathbf{x}^{\top}\mathbf{B}\mathbf{h}+\mathbf{b}^{\top}\mathbf{x}+\mathbf{c}^{\top}\mathbf{x}-\frac{1}{2}\|\mathbf{x}\|^{2}\right),$
where $\mathbf{x}\in\mathbb{R}^{d}$, $\mathbf{h}\in\{\pm1\}^{d_{h}}$
is a random vector of hidden variables, and $Z$ is the normalization
constant. The exact marginal density $p(\mathbf{x})=\sum_{\mathbf{h}\in\{-1,1\}^{d_{h}}}p(\mathbf{x},\mathbf{h})$
is intractable when $d_{h}$ is large, since it involves summing over
$2^{d_{h}}$ terms. Recall that the proposed test only requires the
score function $\nabla_{\mathbf{x}}\log p(\mathbf{x})$ (not the normalization
constant), which can be computed in closed form in this case. In this
problem, $q$ is another RBM where entries of the matrix $\mathbf{B}$
are corrupted by Gaussian noise. This was the problem considered in
\cite{LiuLeeJor2016}. We set $d=50$ and $d_{h}=40$, and generate
samples by $n$ independent chains (i.e., $n$ independent samples)
of blocked Gibbs sampling with 2000 burn-in iterations.

We evaluate the following six kernel-based nonparametric tests with
$\alpha=0.05$, all using the Gaussian kernel.\textbf{ 1. FSSD-rand}:
the proposed FSSD test where the test locations set to random draws
from a multivariate normal distribution fitted to the data. The kernel
bandwidth is set by the commonly used median heuristic i.e., $\sigma_{k}=\mathrm{median}(\{\|\mathbf{x}_{i}-\mathbf{x}_{j}\|,i<j\})$.\textbf{
2. FSSD-opt}: the proposed FSSD test where both the test locations
and the Gaussian bandwidth are optimized (Section \ref{subsec:param_optimization}).
\textbf{3. KSD}: the quadratic-time Kernel Stein Discrepancy test
with the median heuristic.\textbf{ 4. LKS}: the linear-time version
of KSD with the median heuristic. \textbf{5. MMD-opt}: the quadratic-time
MMD two-sample test of \cite{Gretton2012} where the kernel bandwidth
is optimized by grid search to maximize a power criterion as described
in \cite{Sutherland2016}. \textbf{6. ME-opt}: the linear-time mean
embeddings (ME) two-sample test of \cite{Jitkrittum2016a} where parameters
are optimized. We draw $n$ samples from $p$ to run the two-sample
tests (MMD-opt, ME-opt). For FSSD tests, we use $J=5$ (see Section
\ref{sec:pow_vs_J} for an investigation of test power as $J$ varies).
All tests with optimization use 20\% of the sample size $n$ for parameter
tuning. Code is available at \url{https://github.com/wittawatj/kernel-gof}.

Figure \ref{fig:test_powers} shows the rejection rates of the six
tests for the two problems, where each problem is repeated for 200
trials, resampling $n$ points from $q$ every time. In Figure \ref{fig:ex2_laplace}
(Gaussian vs. Laplace), high performance of FSSD-opt indicates that
the test performs well when there are local differences between $p$
and $q$. Low performance of FSSD-rand emphasizes the importance of
the optimization of FSSD-opt to pinpoint regions where $p$ and $q$
differ. The power of KSD quickly drops as the dimension increases,
which can be understood since KSD is the RKHS norm of a function witnessing
differences in $p$ and $q$ across the entire domain, including where
these differences are small. 

\begin{figure}
\centering
\includegraphics[width=0.98\linewidth]{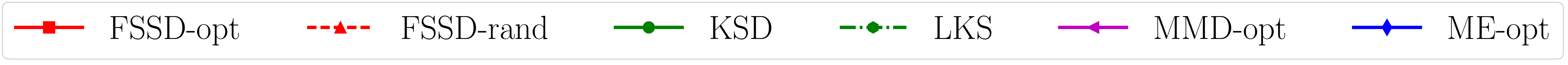} \\
\vspace{-3mm}
\subfloat[Gaussian vs. Laplace. $n=1000$. \label{fig:ex2_laplace}]{
\includegraphics[width=0.223\linewidth]{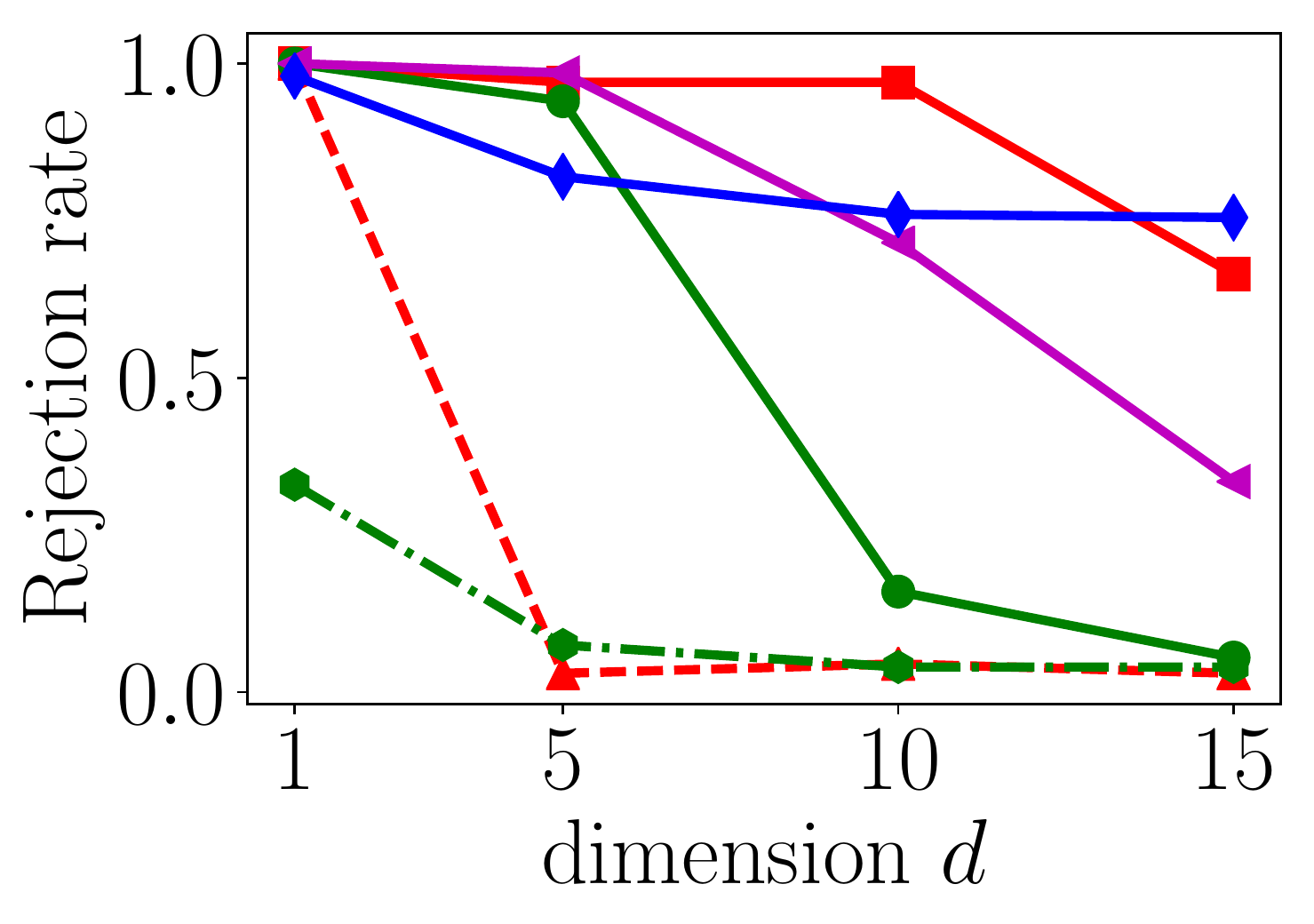}
}%
\,\,
\subfloat[RBM. $n=1000$. Perturb all entries of $\mathbf{B}$. \label{fig:ex2_rbm_dh40}]{ 
\includegraphics[width=0.223\linewidth]{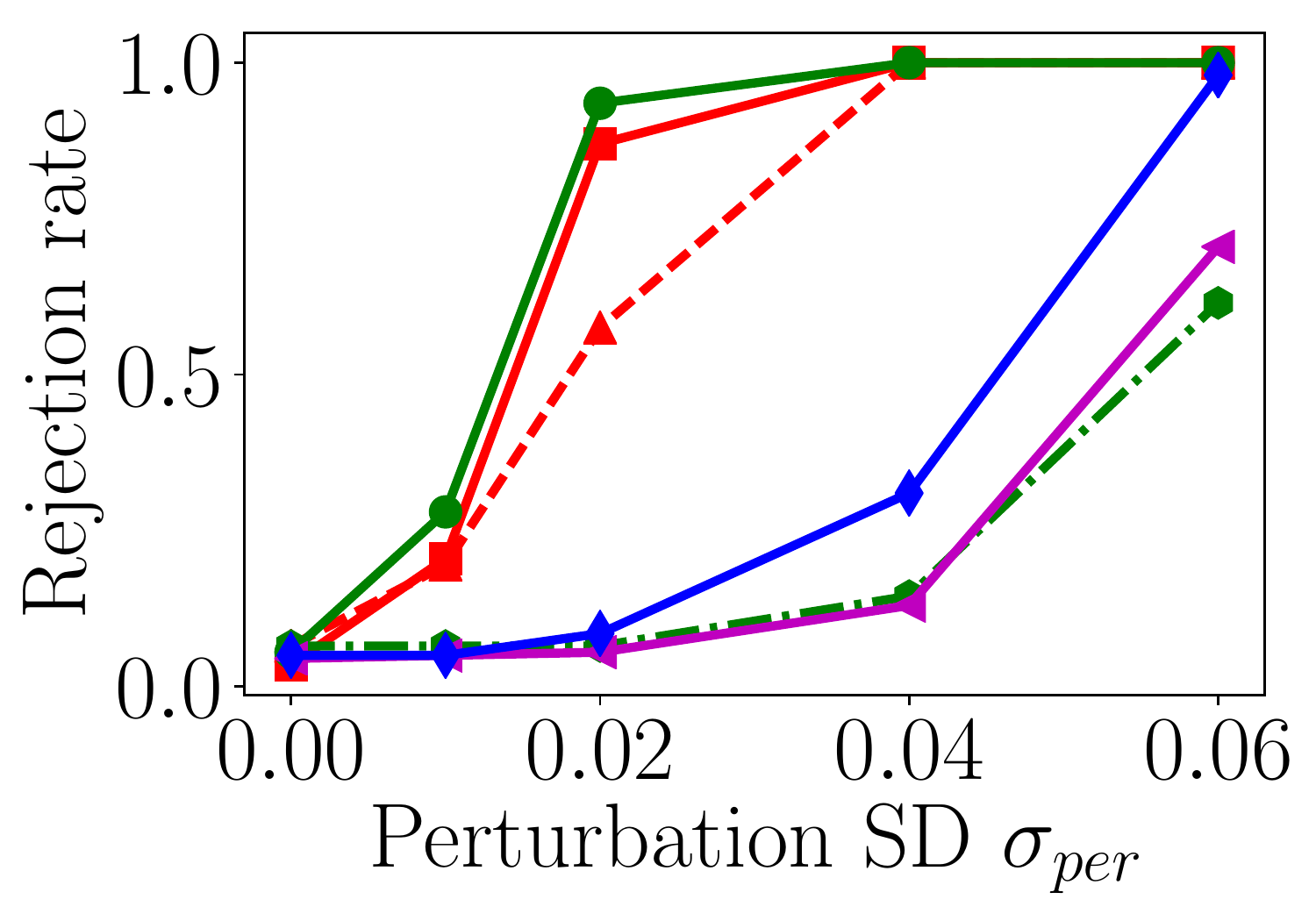} 
}%
\,\,
\subfloat[RBM. $\sigma_{per}=0.1$. Perturb $B_{1,1}$. \label{fig:ex1_rbm_dh40}]{ 
\includegraphics[width=0.232\linewidth]{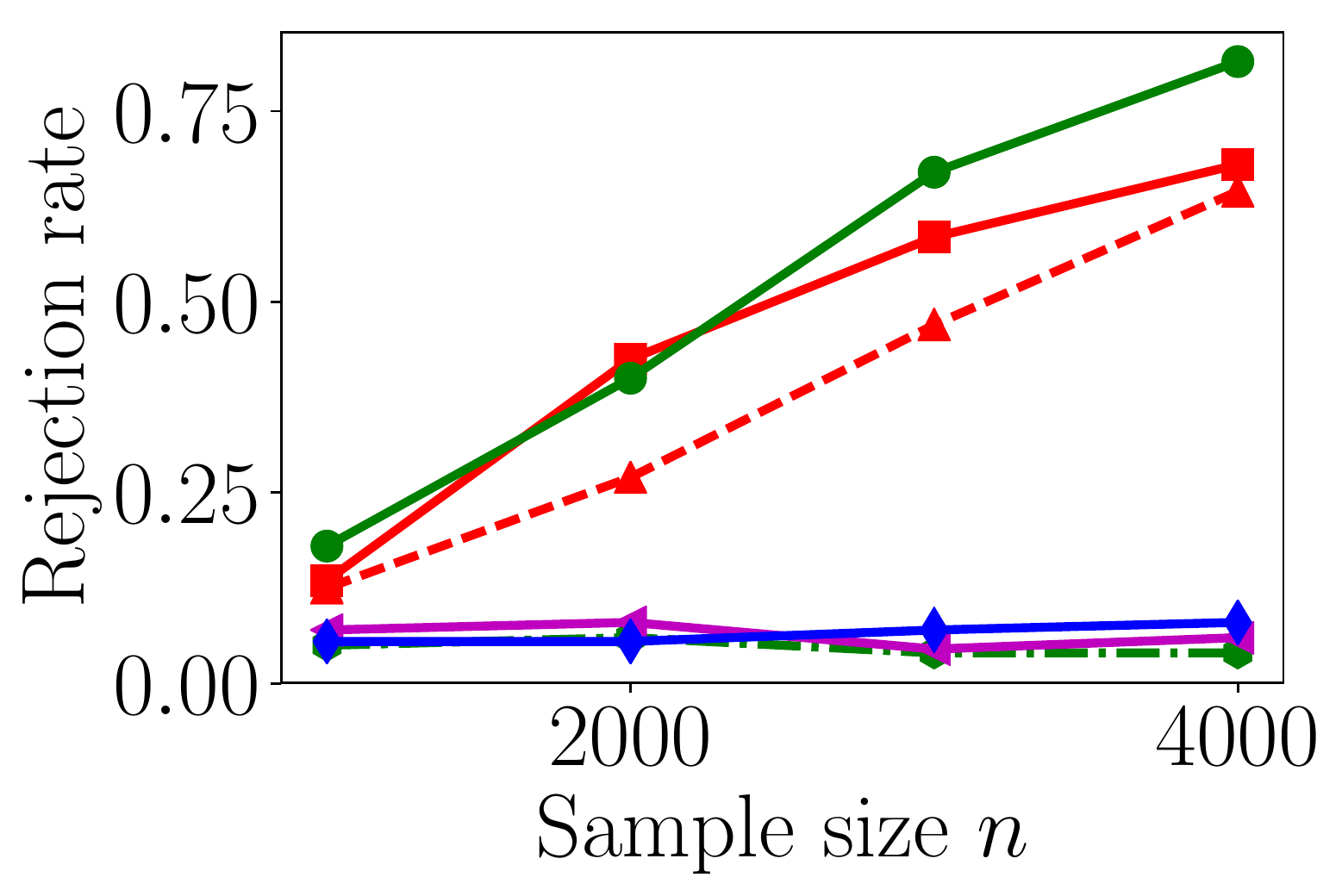} 
}%
\subfloat[Runtime (RBM) \label{fig:ex1_rbm_dh40_time}]{ 
\includegraphics[width=0.235\linewidth]{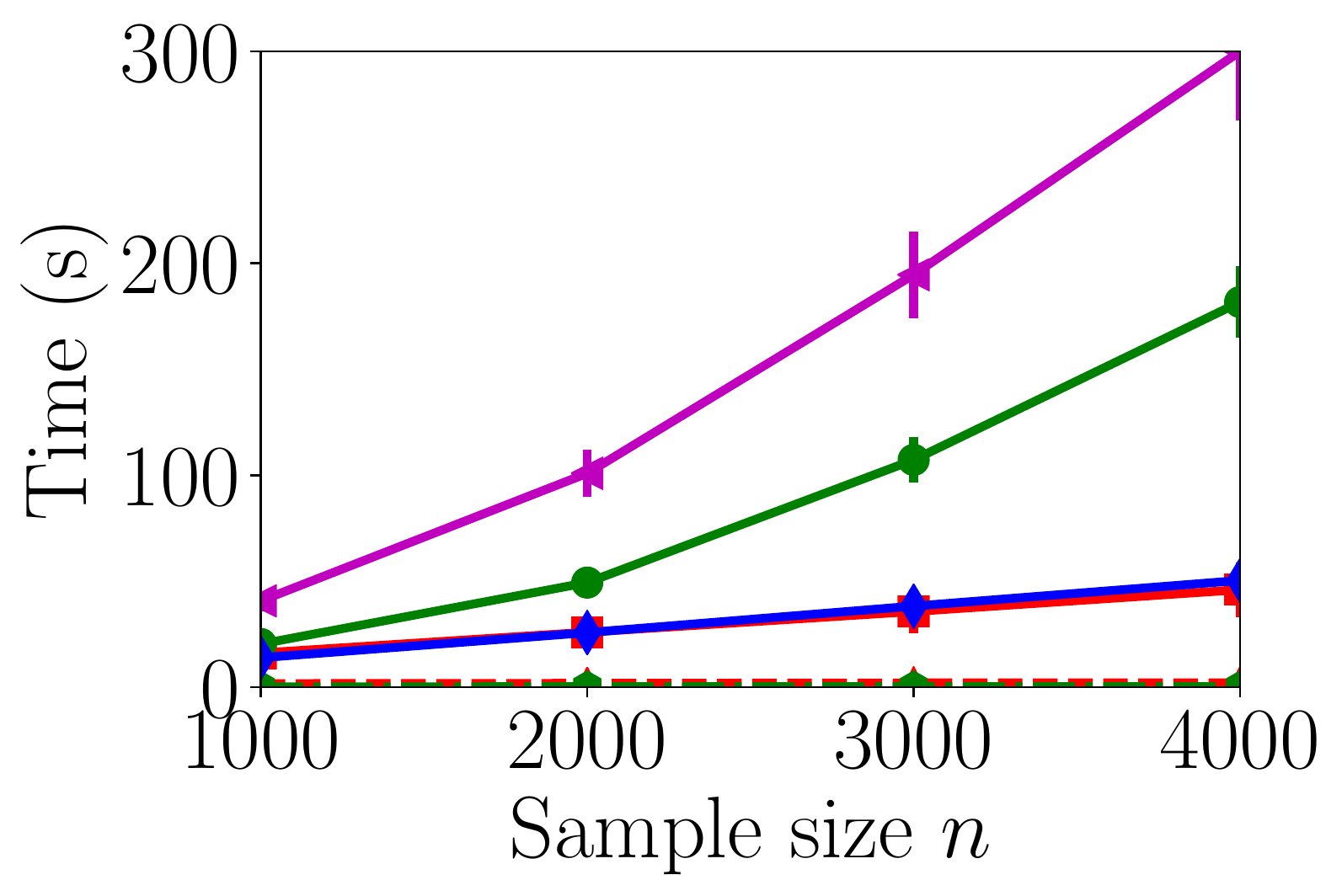} 
} 

\caption{Rejection rates of the six tests. The proposed linear-time FSSD-opt
has a comparable or higher test power in some cases than the quadratic-time
KSD test. \vspace{-4mm}\label{fig:test_powers}}
\end{figure}

We next consider the case of RBMs. Following \cite{LiuLeeJor2016},
$\mathbf{b},\mathbf{c}$ are independently drawn from the standard
multivariate normal distribution, and entries of $\mathbf{B}\in\mathbb{R}^{50\times40}$
are drawn with equal probability from $\{\pm1\}$, in each trial.
The density $q$ represents another RBM having the same $\mathbf{b},\mathbf{c}$
as in $p$, and with all entries of $\mathbf{B}$ corrupted by independent
zero-mean Gaussian noise with standard deviation $\sigma_{per}$.
Figure \ref{fig:ex2_rbm_dh40} shows the test powers as $\sigma_{per}$
increases, for a fixed sample size $n=1000$. We observe that all
the tests have correct false positive rates (type-I errors) at roughly
$\alpha=0.05$ when there is no perturbation noise. In particular,
the optimization in FSSD-opt does not increase false positive rate
when $H_{0}$ holds. We see that the performance of the proposed FSSD-opt
matches that of the quadratic-time KSD at all noise levels. MMD-opt
and ME-opt perform far worse than the goodness-of-fit tests when the
difference in $p$ and $q$ is small ($\sigma_{per}$ is low), since
these tests simply represent $p$ using samples, and do not take advantage
of its structure.

The advantage of having $\mathcal{O}(n)$ runtime can be clearly seen
when the problem is much harder, requiring larger sample sizes to
tackle. Consider a similar problem on RBMs in which the parameter
$\mathbf{B}\in\mathbb{R}^{50\times40}$ in $q$ is given by that of
$p$, where only the first entry $B_{1,1}$ is perturbed by random
$\mathcal{N}(0,0.1^{2})$ noise. The results are shown in Figure \ref{fig:ex1_rbm_dh40}
where the sample size $n$ is varied. We observe that the two two-sample
tests fail to detect this subtle difference even with large sample
size. The test powers of KSD and FSSD-opt are comparable when $n$
is relatively small. It appears that KSD has higher test power than
FSSD-opt in this case for large $n$. However, this moderate gain
in the test power comes with an order of magnitude more computation.
As shown in Figure \ref{fig:ex1_rbm_dh40_time}, the runtime of the
KSD is much larger than that of FSSD-opt, especially at large $n$.
In these problems, the performance of the new test (even without optimization)
far exceeds that of the LKS test. Further simulation results can be
found in Section \ref{sec:more_expr}. 

\begin{wrapfigure}{r}{0.64\textwidth} 
\vspace{-8mm} 
\subfloat[$p=$ 2-component GMM. \label{fig:low_gmm_obj}]{ 
\includegraphics[width=0.43\linewidth]{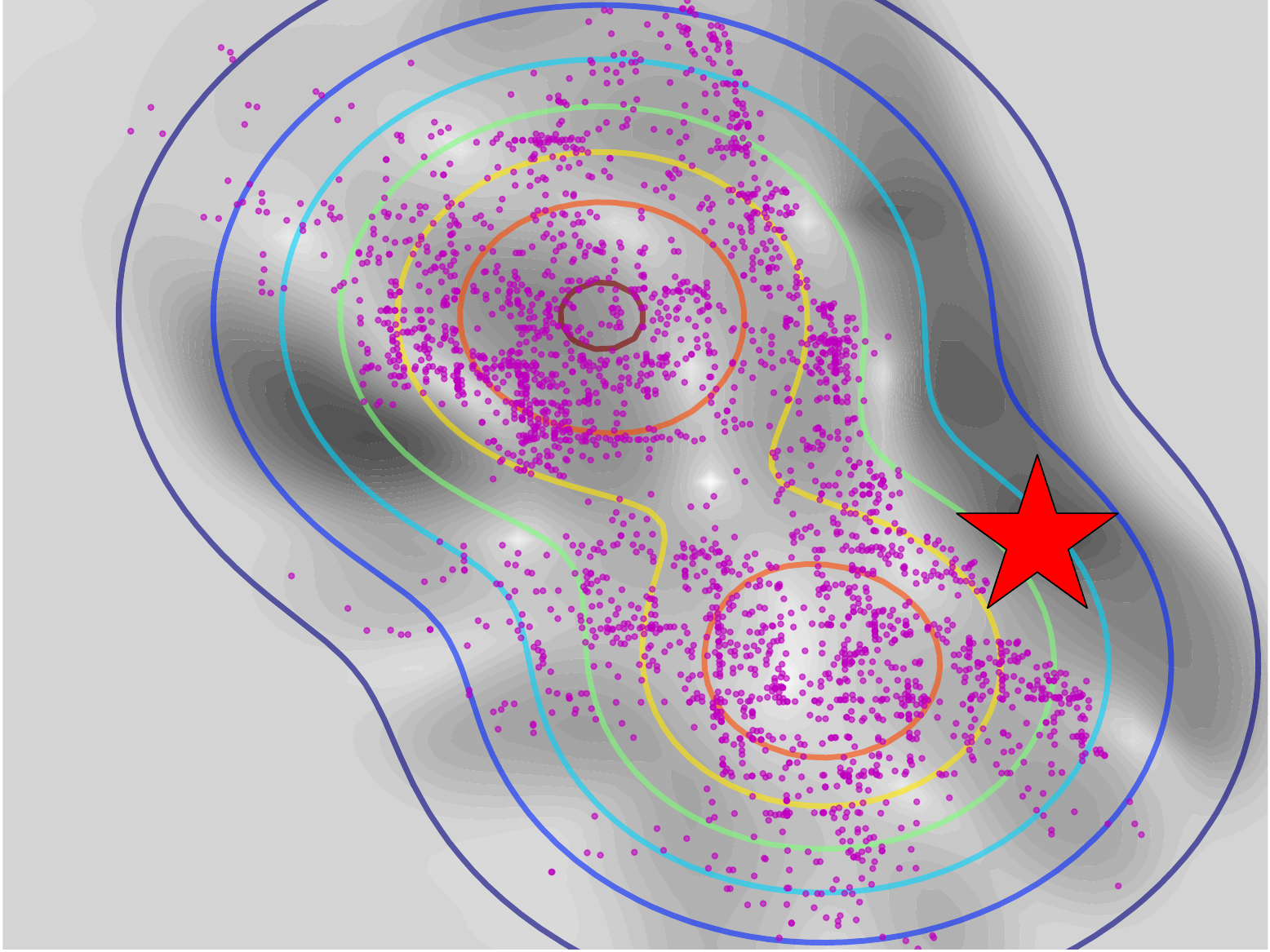} 
}
\subfloat[$p=$ 10-component GMM \label{fig:high_gmm_obj}]{ 
\includegraphics[width=0.55\linewidth]{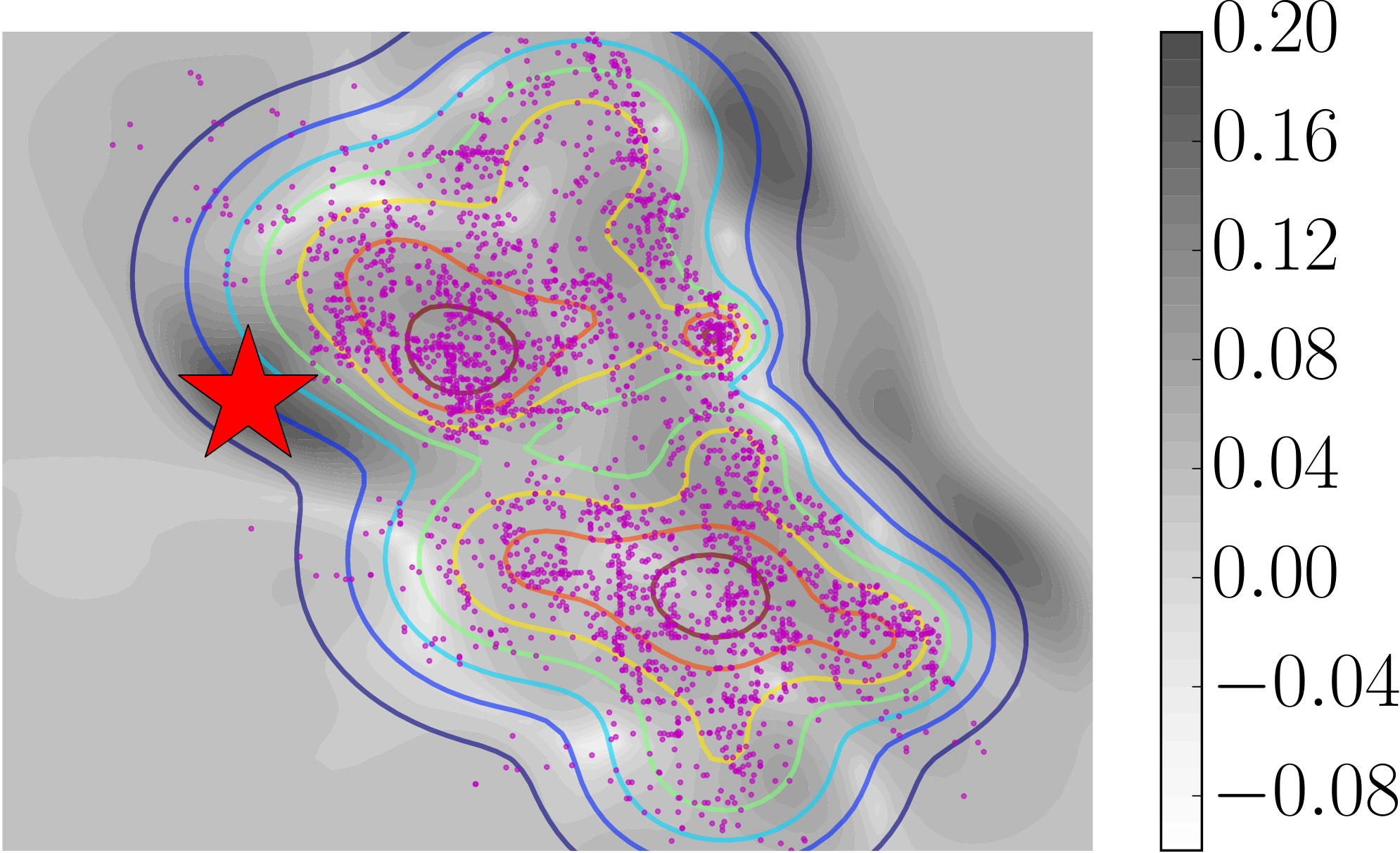} 
}
\caption{Plots of the optimization objective as a function of test location $\mathbf{v} \in \mathbb{R}^2$ in  the Gaussian mixture model (GMM) evaluation task.
}
\label{fig:mixture_obj}
\vspace{-2mm}
\end{wrapfigure} 

\paragraph{Interpretable Features }

In the final simulation, we demonstrate that the learned test locations
are informative in visualising where the model does not fit the data
well. We consider crime data from the Chicago Police Department, recording
$n=11957$ locations (latitude-longitude coordinates) of robbery events
in Chicago in 2016.\footnote{Data can be found at \url{https://data.cityofchicago.org}.}
We address the situation in which a model $p$ for the robbery location
density is given, and we wish to visualise where it fails to match
the data. We fit a Gaussian mixture model (GMM) with the expectation-maximization
algorithm to a subsample of 5500 points. We then test the model on
a held-out test set of the same size to obtain proposed locations
of relevant features $\mathbf{v}$. Figure \ref{fig:low_gmm_obj}
shows the test robbery locations in purple, the model with two Gaussian
components in wireframe, and the optimization objective for $\mathbf{v}$
as a grayscale contour plot (a red star indicates the maximum). We
observe that the 2-component model is a poor fit to the data, particularly
in the right tail areas of the data, as indicated in dark gray (i.e.,
the objective is high). Figure \ref{fig:high_gmm_obj} shows a similar
plot with a 10-component GMM. The additional components appear to
have eliminated some mismatch in the right tail, however a discrepancy
still exists in the left region. Here, the data have a sharp boundary
on the right side following the geography of Chicago, and do not exhibit
exponentially decaying Gaussian-like tails.  We note that tests based
on a learned feature located at the maximum both correctly reject
$H_{0}$. 

\newpage{}

\subsection*{Acknowledgement}

WJ, WX, and AG thank the Gatsby Charitable Foundation for the financial
support. ZSz was financially supported by the Data Science Initiative.
KF has been supported by KAKENHI Innovative Areas 25120012.


\bibliographystyle{abbrvnat}
\bibliography{kgof}

\newpage
\appendix

\begin{center}
{\Large{}\ourtitle{}}
\par\end{center}{\Large \par}

\begin{center}
\textcolor{black}{\Large{}Supplementary}
\par\end{center}{\Large \par}

\section{Rejection Rate vs. Number of Test Locations $J$}

\label{sec:pow_vs_J}
\begin{figure}[h]
\hspace{-2mm}
\subfloat[SG. $d=5$. $\alpha=0.05$ \label{fig:ex3_sg}]{
\includegraphics[width=0.28\linewidth]{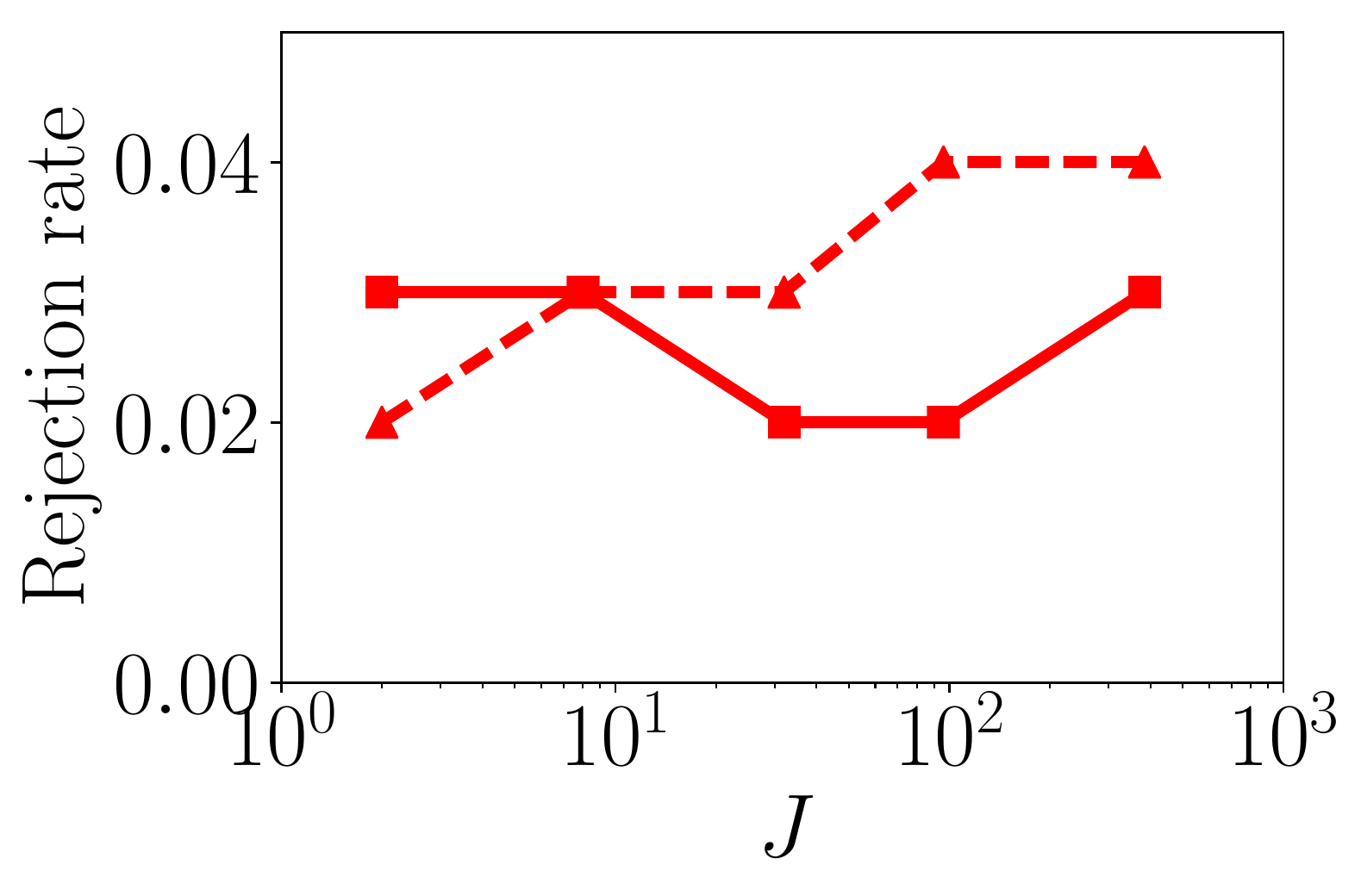}
}
\hspace{-2mm}
\subfloat[Gaussian vs. GMM. $d=1$. \label{fig:ex3_g_vs_gmm}]{ 
\includegraphics[width=0.27\linewidth]{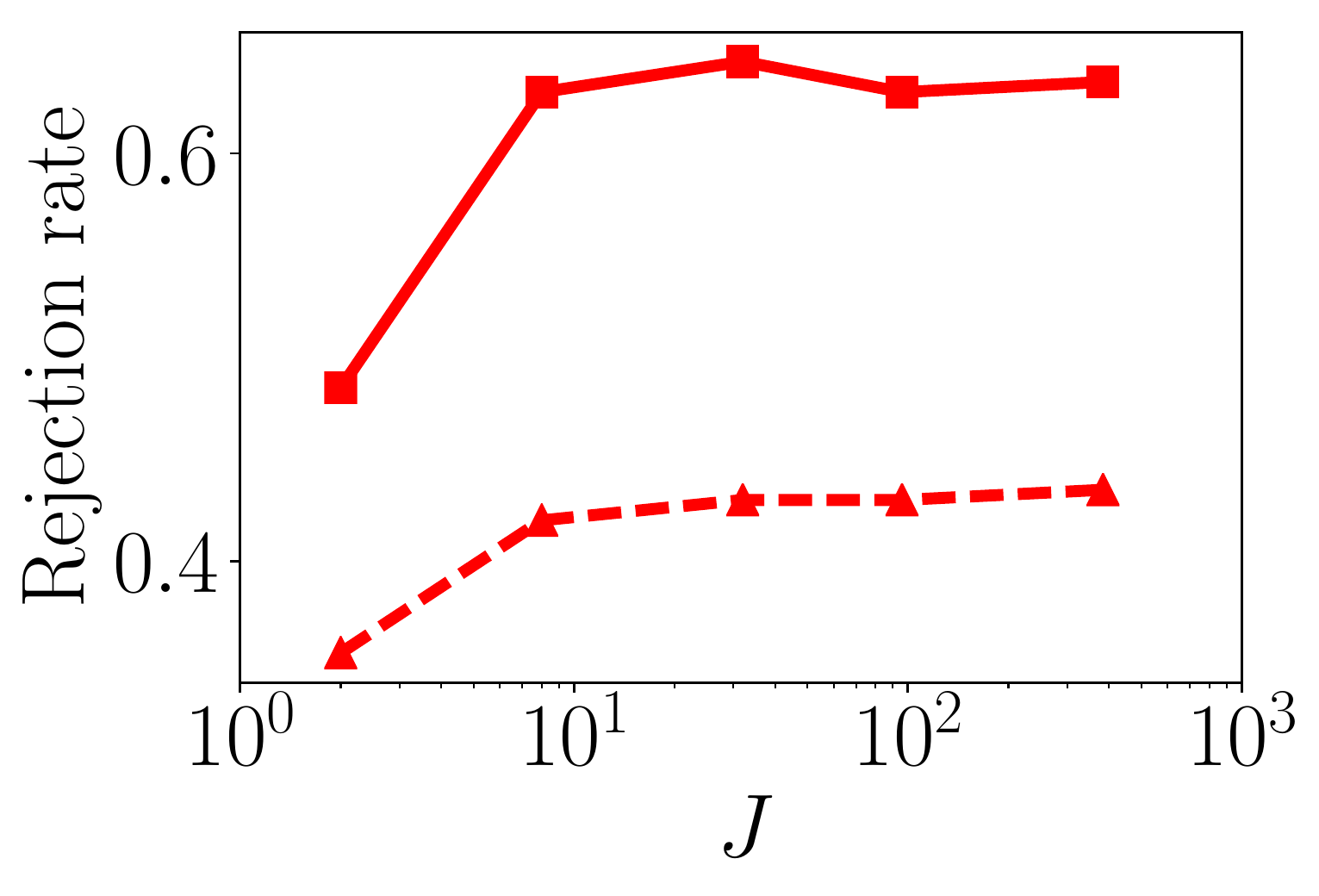} 
}\hspace{-2mm}
\subfloat[GVD. $d=5$. \label{fig:ex3_gvd5}]{ 
\includegraphics[width=0.42\linewidth]{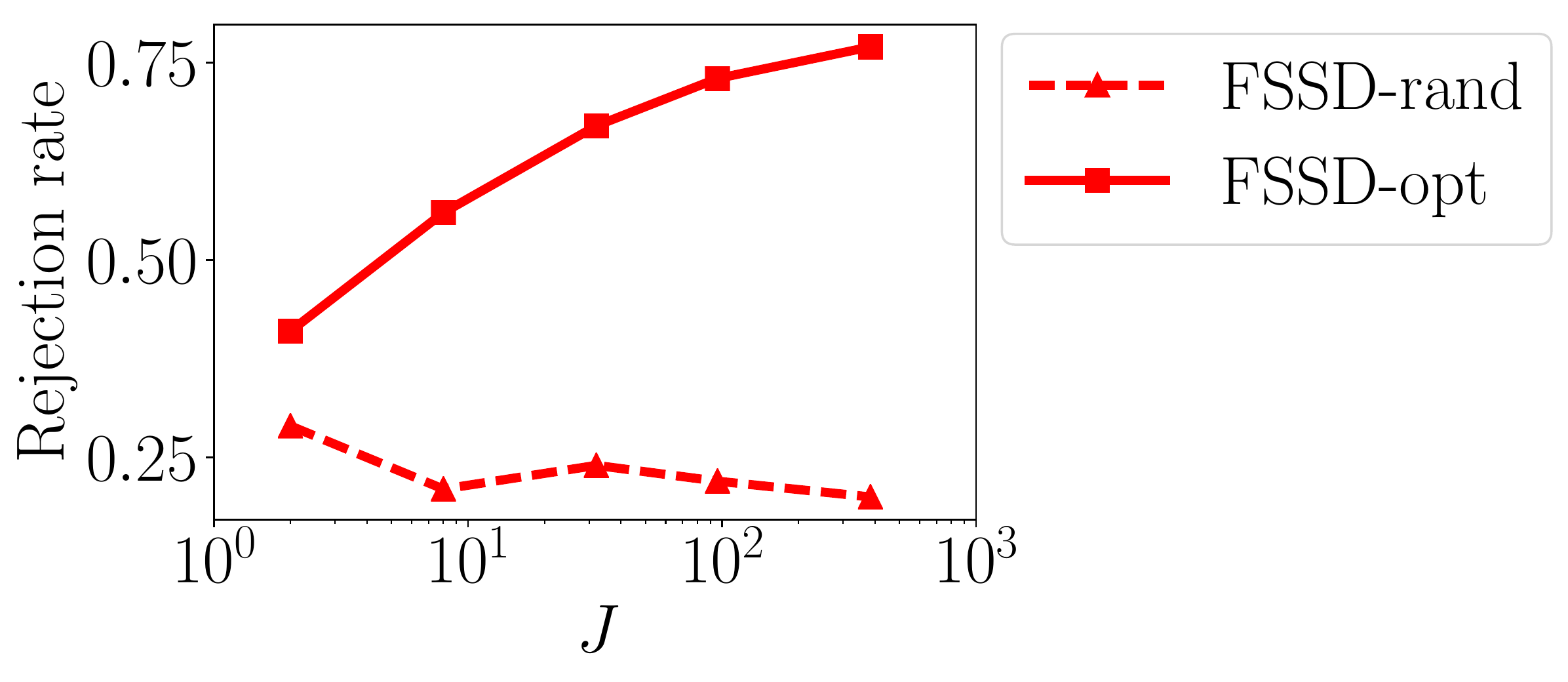} 
} 

\caption{Plots of rejection rate against the number of test locations $J$
in the three toy problems in Section \ref{sec:pow_vs_J}. \label{fig:pow_vs_J}}
\end{figure}

The aim of this section is to explore the test power of the proposed
FSSD test as a function of the number of test locations $J$. We consider
three synthetic problems to illustrate three phenomena depending on
the characteristic of the problem. We note that the test power may
not necessarily increase with $J$. Figure \ref{fig:pow_vs_J} shows
the rejection rate as a function of the test locations $J$ in the
three problems described below. In all cases, the sample size is set
to $n=500$, the train/test ratio is 50\%, and the significance level
is $\alpha=0.05$. All rejection rates are computed with 200 trials
with data sampled from the specified $q$ in every trial. 

We emphasize that the FSSD test is not designed to be used with large
$J$, since doing so defeats the purpose of a linear-time test. We
show in the main text in Section \ref{fig:test_powers} that using
$J=5$ is typically sufficient in practice.

\paragraph{Same Gaussian (SG): }

In this problem, \textbf{$p=q=\mathcal{N}(\mathbf{0},\mathbf{I})$
}in $\mathbb{R}^{5}$ i.e., $H_{0}$ is true. It can be seen in Figure
\ref{fig:ex3_sg} that both the FSSD tests with and without optimization
achieve correct false positive rate at roughly $\alpha$ for all $J$
considered. That is, under $H_{0}$, the false rejection rate stays
at the right level for all $J$.

\paragraph{Gaussian vs. Gaussian mixture model (GMM): }

This is a one-dimensional problem where $p=\mathcal{N}(0,1)$ and
$q=0.9\mathcal{N}(0,1)+0.1\mathcal{N}(0,0.1^{2})$ i.e., a mixture
of two normal distributions. In this problem, $p$ significantly differs
from $q$ in a small region around 0. This difference is created by
the second mixture component. The characteristic of this problem is
the local difference of $p$ and $q$.

Figure \ref{fig:ex3_g_vs_gmm} indicates that using random test locations
(FSSD-rand) does not give high test power. With optimization (FSSD-opt),
the power increases as $J$ increases up to a point, after which it
slightly drops down and reaches a plateau. This behavior can be explained
by noting that there is only a very small region around 0 to detect
the difference. More signal can be gained with diminishing return
by increasing the number of test locations around 0. When $J$ is
sufficiently high, the increase in the variance of the statistic outweighs
the gain of the signal (recall that the variance of the null distribution
increases with $J$). This increase in the variance reduces the test
power.

\paragraph{Gaussian Variance Difference (GVD): }

This is a synthetic problem studied in \cite{Jitkrittum2016a} where
$p=\mathcal{N}(\mathbf{0},\mathbf{I})$ and $q=\mathcal{N}(\mathbf{0},\diag(2,1\ldots,1))$
in $\mathbb{R}^{5}$. In this case, the region of difference between
$q$ and $p$ exists only along the first dimension, and is broad.

In this case, Figure \ref{fig:ex3_gvd5} shows that, with optimization,
the power increases as the number of test locations increases. Unlike
the case of Gaussian vs. GMM, the region of difference in this case
is broad, and can accommodate more test locations to increase the
 signal. Despite this, we expect the test power to reach a plateau
when $J$ is sufficiently large for the same reason as described previously.
In FSSD-rand, random test locations decrease the power due to the
increase in the variance. Since only one dimension is relevant in
determining the difference of $p$ and $q$, it is unlikely that random
locations are in the right region. 

\section{More Experiments }

\begin{figure}[th]
\includegraphics[width=0.98\linewidth]{img/legend6_hori.pdf} \\[-3mm]
\centering
\subfloat[RBM. $n=1000$. Perturb all entries of $\mathbf{B}$. \label{fig:ex2_rbm_dh10}]{
\includegraphics[width=0.225\linewidth]{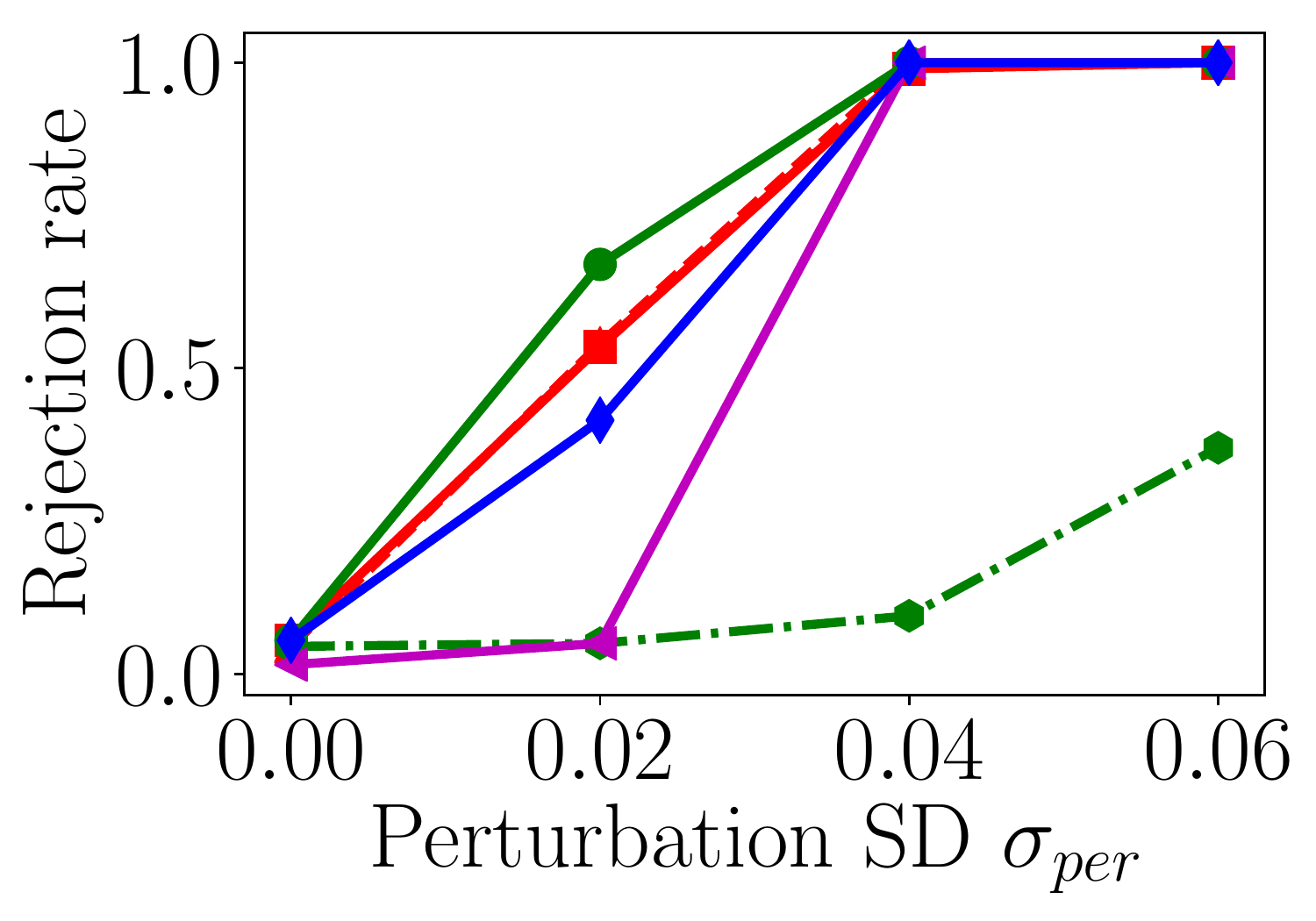}
}
\hspace{1mm}
\subfloat[RBM. $\sigma_{per}=0.1$. Perturb $B_{1,1}$. \label{fig:ex1_rbm_dh10}]{ 
\includegraphics[width=0.23\linewidth]{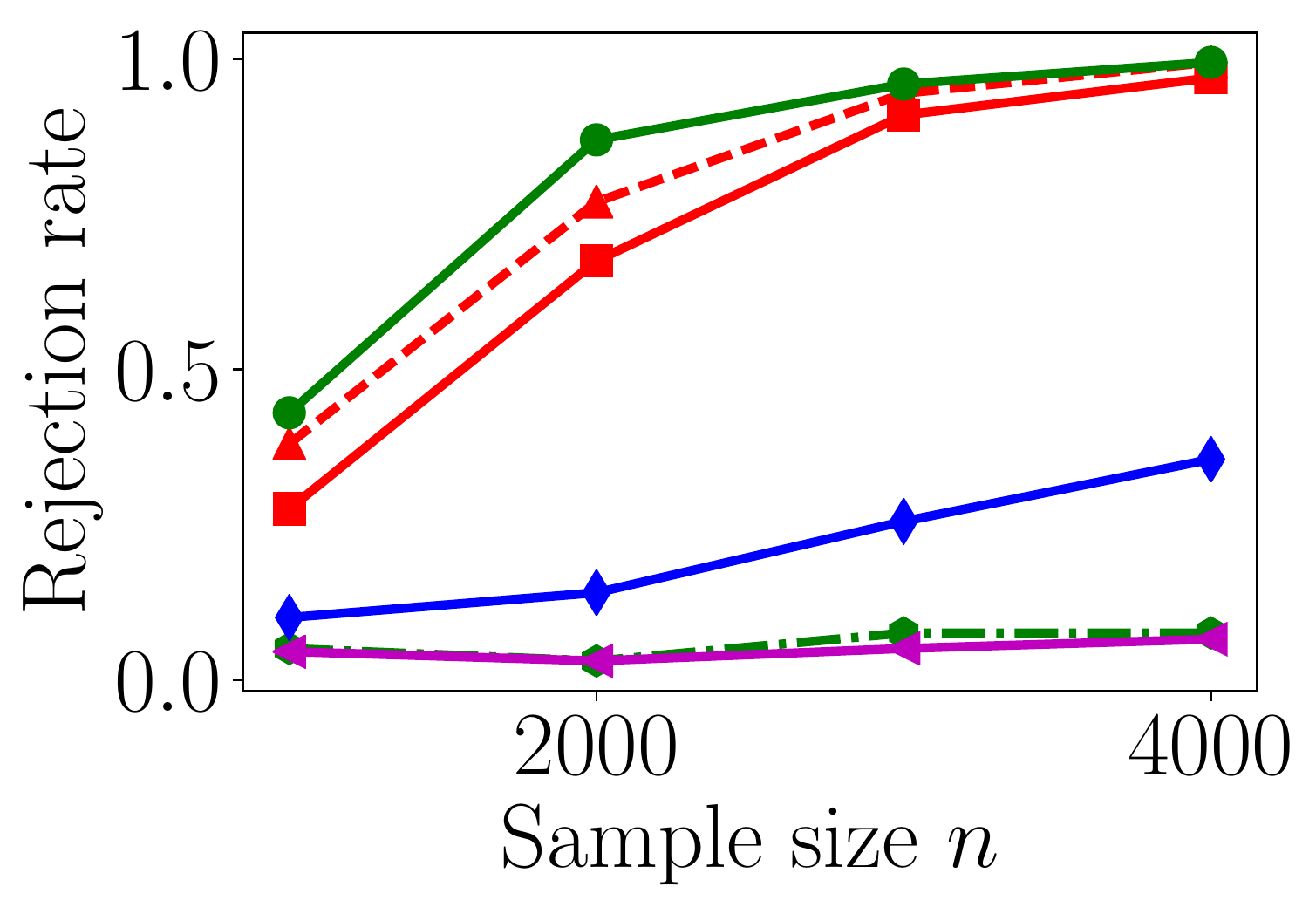} 
}
\hspace{0mm}
\subfloat[Runtime (RBM) \label{fig:ex1_rbm_dh10_time}]{ 
\includegraphics[width=0.24\linewidth]{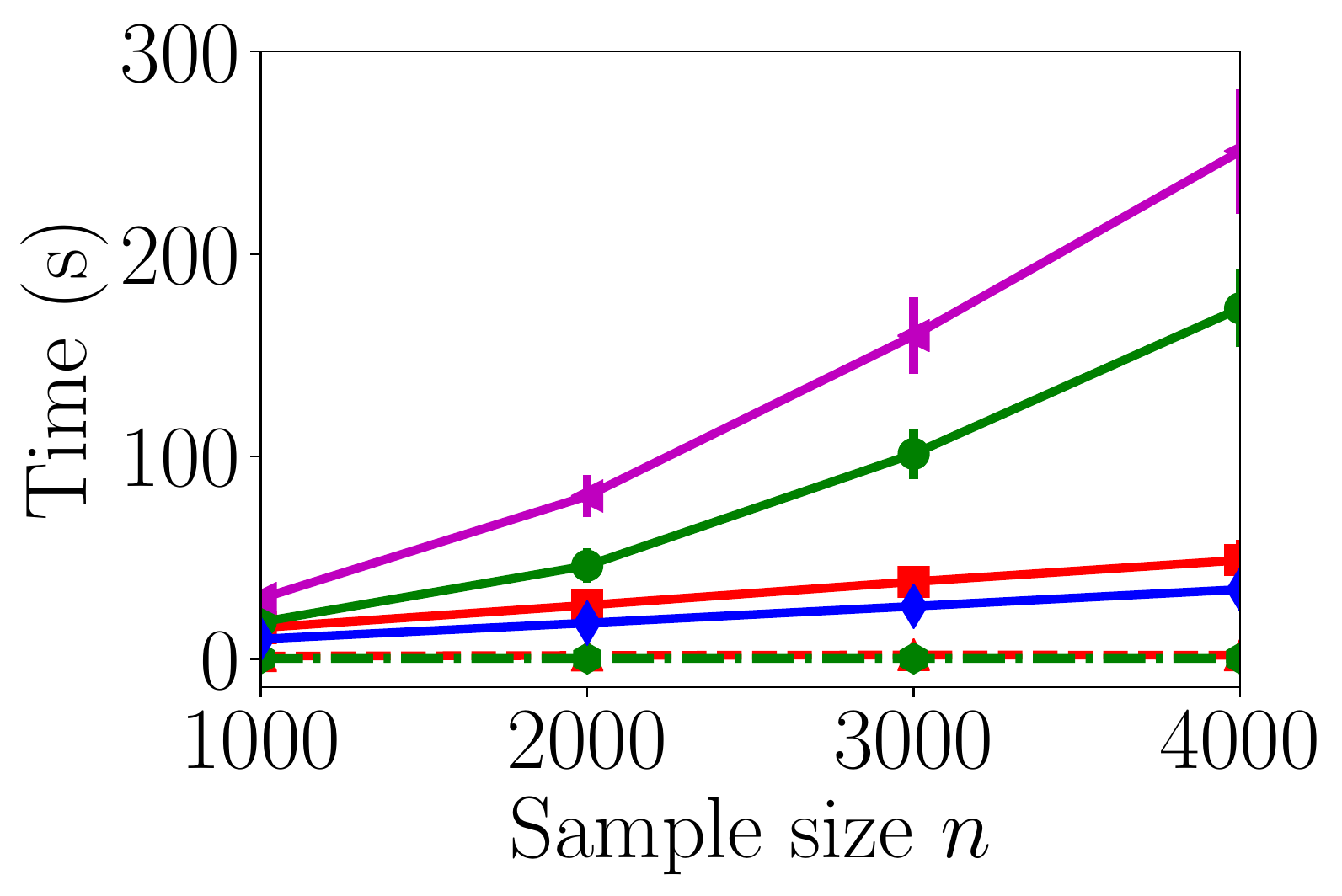} 
}
\subfloat[RBM. No perturbation. $H_0$ holds. \label{fig:ex1_rbm_dh10_h0}]{ 
\includegraphics[width=0.235\linewidth]{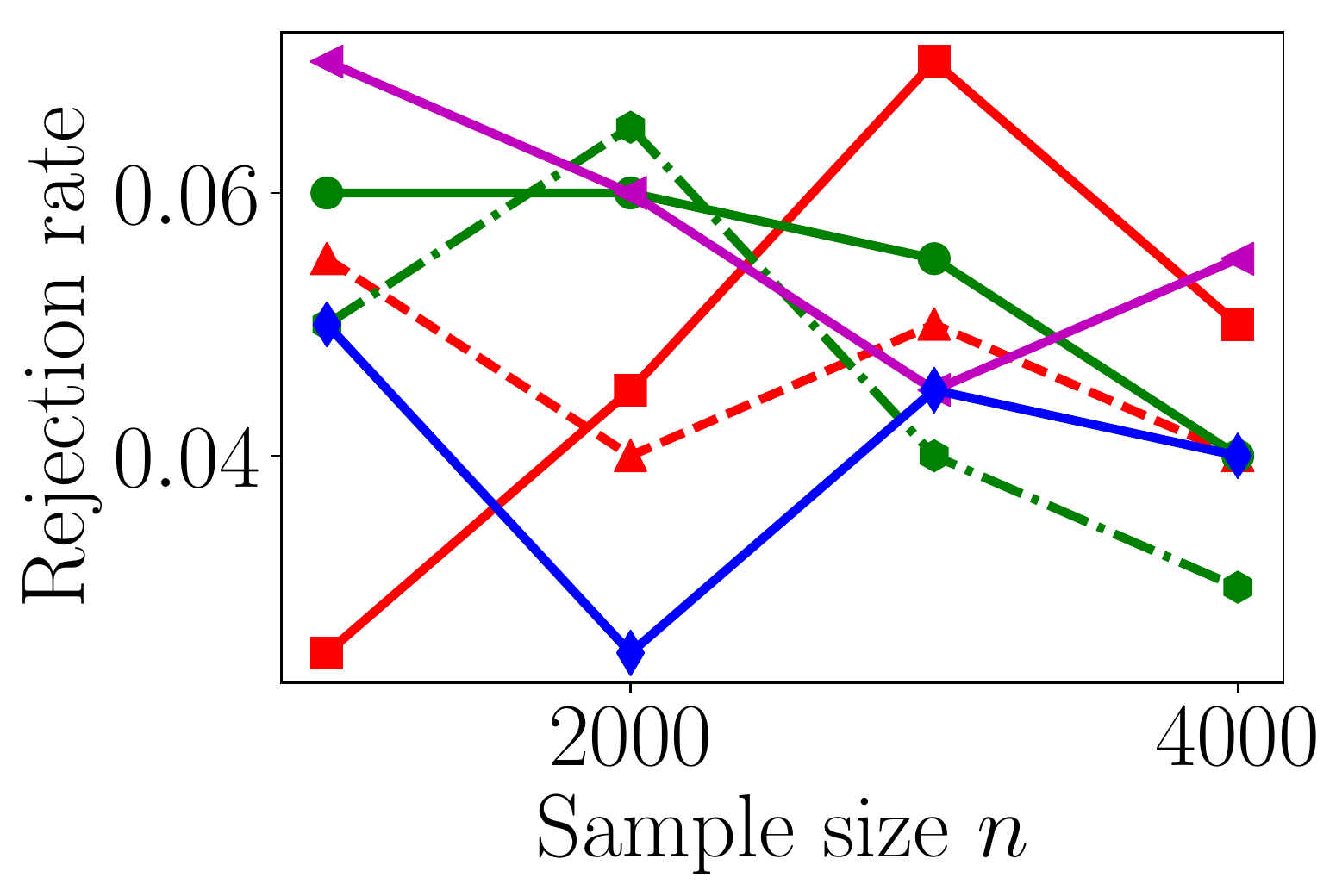} 
} 

\caption{Rejection rates of the six tests in the RBM problem with $d=50$ and
$d_{h}=10$. \label{fig:test_powers_more}}
\end{figure}
\begin{figure}[th]
\begin{center}
\subfloat[$d=50, d_h=10$ \label{fig:rbm_marginals_d50_dh10}]{ 
\includegraphics[width=0.35\linewidth]{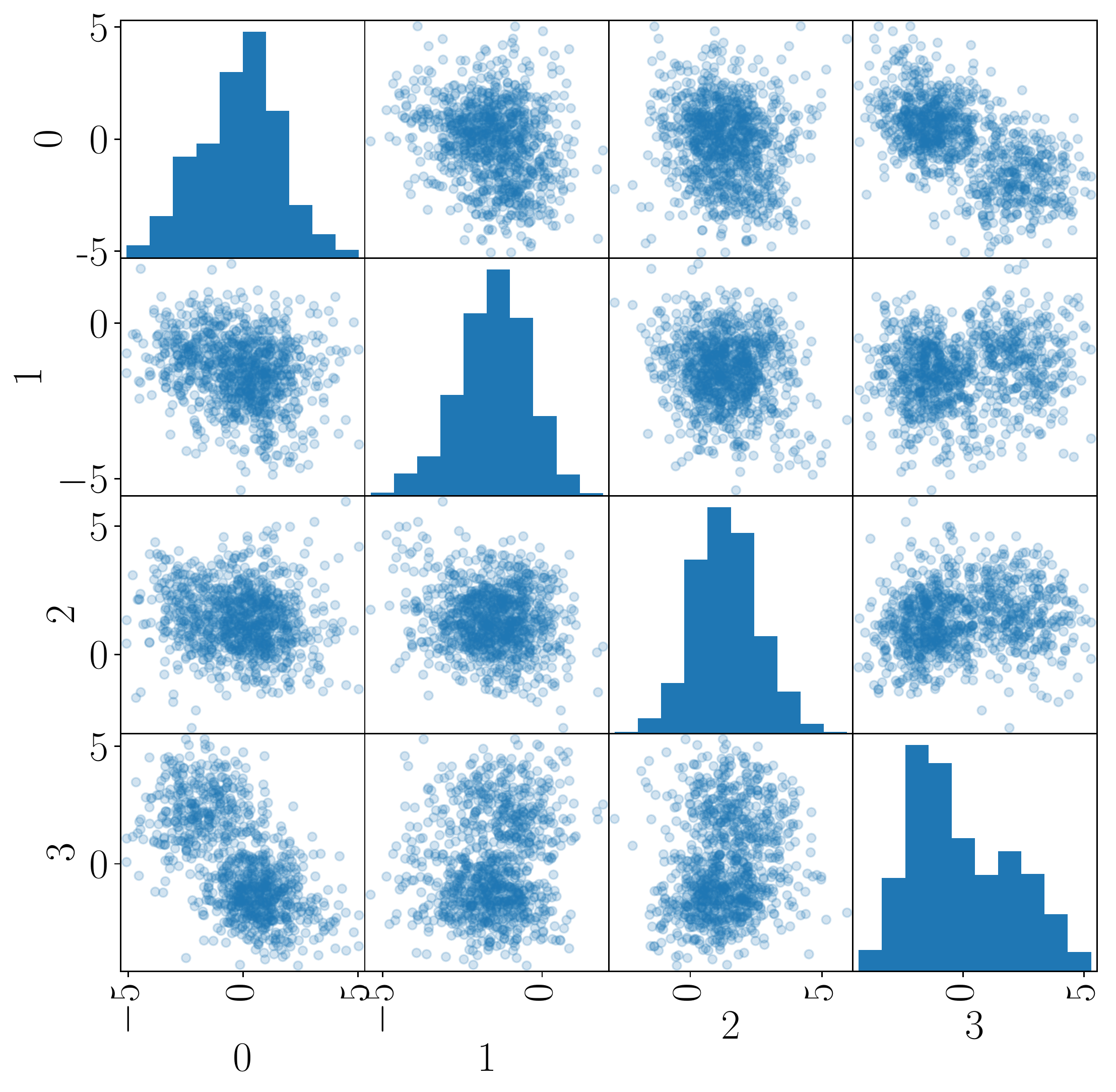} 
}
\hspace{3mm}
\subfloat[$d=50, d_h=40$ \label{fig:rbm_marginals_d50_dh40}]{ 
\includegraphics[width=0.35\linewidth]{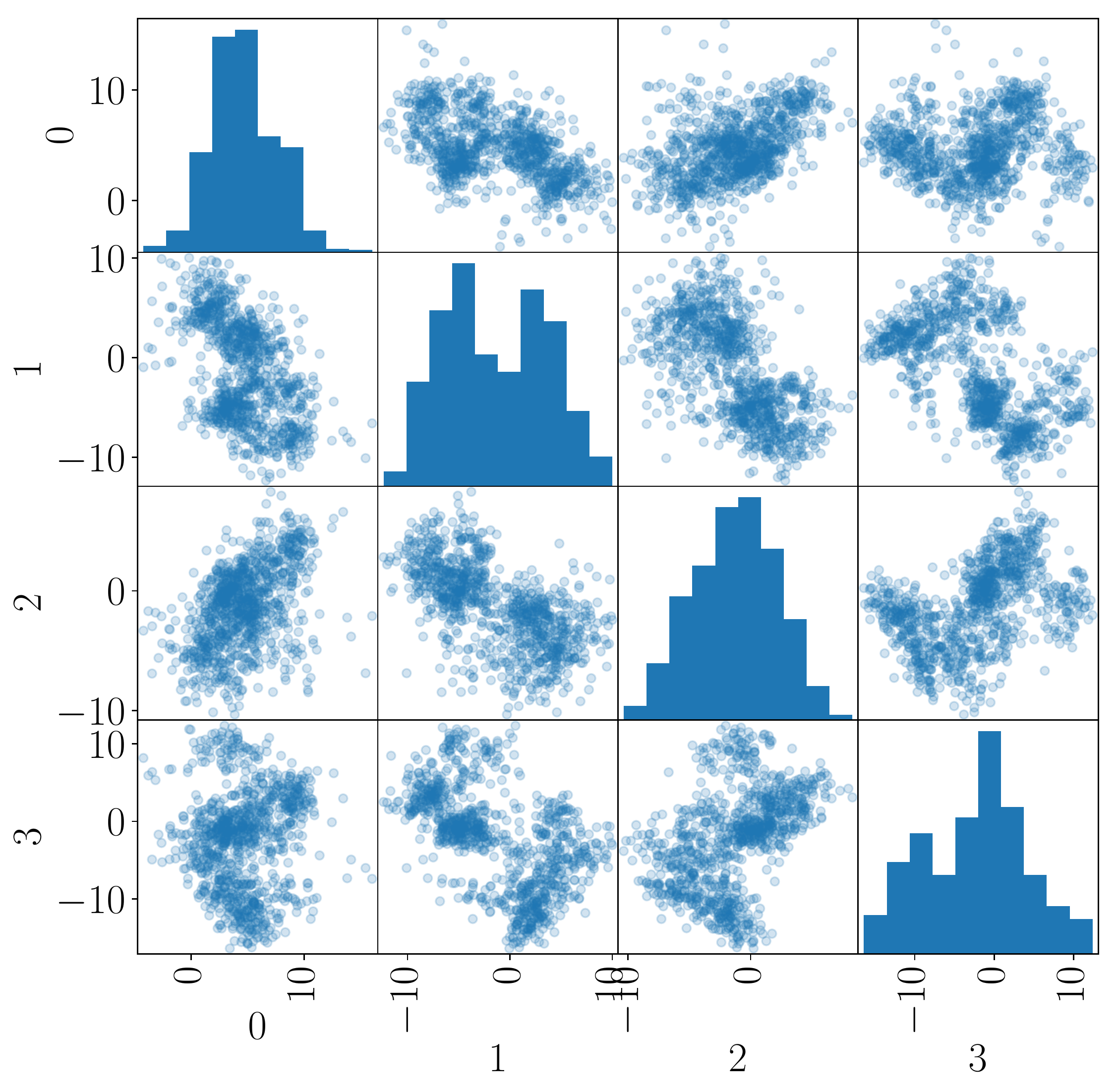} 
}
\caption{Pairwise scatter plots of 1000 points drawn from RBMs. Only the first 4 variates out of 50 are shown. \textbf{(a)}: RBM with $d=50$ dimensions with $d_h=10$ latent variables. \textbf{(b)}: RBM with $d=50$ dimensions with $d_h=40$ latent variables.
}
\end{center}
\label{fig:rbm_marginals}
\end{figure}

\label{sec:more_expr} Recall that in Section \ref{sec:experiments},
we evaluate the test powers of all the six tests on the RBM problem
with $d=50$ and $d_{h}=40$ (i.e., the number of latent variables).
We aim to provide more evaluations in this section. In \cite{LiuLeeJor2016},
the setting of $d=50$ and $d_{h}=10$ was studied. Here we consider
the same setting and show the results in Figure \ref{fig:test_powers_more}
where all other problem configurations are the same as in Section
\ref{sec:experiments}. 

In Figure \ref{fig:ex2_rbm_dh10}, $p$ is set to an RBM with parameters
randomly drawn (described in Section \ref{sec:experiments}), and
$q$ is the same RBM with all entries of the parameter $\mathbf{B}\in\mathbb{R}^{50\times10}$
perturbed by independent Gaussian noise with standard deviation $\sigma_{per}$,
which varies from 0 to 0.06. We observe that the proposed FSSD-opt
and KSD perform comparably. Figure \ref{fig:ex1_rbm_dh10} considers
a hard problem where only the first entry $B_{1,1}$ is perturbed
by noise following $\mathcal{N}(0,0.1^{2})$, and the sample size
$n$ is varied. In both of these two cases, the overall trend is similar
to the case of $d=50$ and $d_{h}=40$ presented in Figure \ref{fig:test_powers}.
It is interesting to note that FSSD-rand, relying on random test locations,
performs comparably or even outperforms FSSD-opt in the case of $d=50,d_{h}=10$,
but not in the case of $d=50,d_{h}=40$. This phenomenon can be explained
as follows. In the case of $d=50,d_{h}=10$, the data generated from
the RBM tend to have simple structure (see Figure \ref{fig:rbm_marginals_d50_dh10}).
By contrast, data generated from the RBM with $d=50,d_{h}=40$ (more
latent variables) have larger variance, and can form a complicated
structure (Figure \ref{fig:rbm_marginals_d50_dh40}), requiring a
careful choice of test locations to detect differences of $p$ and
$q$. When $d=50,d_{h}=10$, however, random test locations given
by random draws from a Gaussian distribution fitted to the data are
sufficient to capture the simple structural difference. This explains
why FSSD-rand can perform well in this case. Additionally, FSSD-rand
also has 20\% more testing data, since FSSD-opt uses 20\% of the sample
for parameter tuning.

Figure \ref{fig:ex1_rbm_dh10_h0} shows the rejection rates of all
the tests as the sample size increases when $p$ and $q$ are the
same RBM. All the tests have roughly the right false rejection rates
at the set significance level $\alpha=0.05$. 

\section{Proof of Theorem \ref{thm:fssd} }

Recall Theorem \ref{thm:fssd}: \fssd*
\begin{proof}
Since $k$ is real analytic, the components $g_{1},\ldots,g_{d}$
of $\mathbf{g}$ are real analytic by Lemma \ref{lem:rkhs_analytic}.
For each $i=1,\ldots,d$, if $g_{i}$ is real analytic, then $\sum_{j=1}^{J}g_{i}^{2}(\mathbf{v}_{j})=0$
if and only if $g_{i}(\mathbf{y})=0$ for all $\mathbf{y}\in\mathcal{X}$,
$\eta$-almost surely (require that the domain $\mathcal{X}$ be a
connected open set) \cite{Mityagin2015}. This implies that $\frac{1}{dJ}\sum_{i=1}^{d}\sum_{j=1}^{J}g_{i}^{2}(\mathbf{v}_{j})=0$
if and only if $\mathbf{g}(\mathbf{y})=\mathbf{0}$ for all $\mathbf{y}\in\mathcal{X}$,
$\eta$-almost surely. By Theorem \ref{thm:stein_iff_pq}, $\mathbf{g}=\mathbf{0}$
(the zero function) if and only if $p=q$.
\end{proof}

\section{More on Bahadur Slope}

\label{sec:bahadur_details}In practice, the main difficulty in determining
the approximate Bahadur slope is the computation of $-2\plim_{n\to\infty}\frac{\log(1-F(T_{n}))}{\rho(n)}$,
typically requiring the aid of the theory of large deviations. There
are further sufficient conditions which make the computation easier.
The following conditions are due to \cite{Gle1964,Gle1966}, first
appearing in \cite{Bah1960} in a slightly less general form.
\begin{defn}
\label{def:nulldist_class}Let $\mathcal{D}(a,t)$ be a class of all
continuous cumulative distribution functions (CDF) $F$ such that
$-2\log(1-F(x))=ax^{t}(1+o(1)),$ as $x\to\infty$ for $a>0$ and
$t>0$.
\end{defn}
\begin{thm}[\cite{Gle1964,Gle1966}]
 \label{thm:gleser1964_slope}Consider a sequence of test statistic
$T_{n}$. Assume that
\begin{enumerate}
\item There exists a function $F(x)$ such that for $\theta\in\Theta_{0}$,
$\lim_{n\to\infty}P_{\theta}(T_{n}<x)=F(x)$, for all $x$, and such
that $F\in\mathcal{D}(a,t)$ for some $a>0$ and $t>0$ (see Definition
\ref{def:nulldist_class}).
\item There exists a continuous, strictly increasing function $R:(0,\infty)\to(0,\infty)$
with $\lim_{n\to\infty}R(n)=\infty$, and a function $b(\theta)$
with $0<b(\theta)<\infty$ defined on $\Theta\backslash\Theta_{0}$,
such that for all $\theta\in\Theta\backslash\Theta_{0}$, $\plim_{n\to\infty}T_{n}/R(n)=b(\theta).$
\end{enumerate}
Then, $-2\plim_{n\to\infty}\frac{\log(1-F(T_{n}))}{\left[R(n)\right]^{t}}=a\left[b(\theta)\right]^{t}=:c(\theta),$
the approximate slope of the sequence $T_{n}$, where $\rho(n)=R(n)^{t}$
(see Section \ref{sec:bahadur_efficiency}).
\end{thm}
\begin{thm}[\cite{Gle1964,Gle1966}]
\label{thm:gleser1964_efficiency} Consider two sequences of test
statistics $T_{n}^{(1)}$ and $T_{n}^{(2)}$. Let $F^{(i)}$ be the
CDF of $T_{n}^{(i)}$ for $i=1,2$. Assume that each sequence satisfies
all the conditions in Theorem \ref{thm:gleser1964_slope} with $F^{(i)}\in\mathcal{D}(a_{i},t_{i})$.
Further, assume that $\left[R^{(1)}(x)\right]^{t_{1}}=\left[R^{(2)}(x)\right]^{t_{2}}$
for all $x$. Then 
\begin{align*}
\plim_{n\to\infty}\frac{\log(1-F^{(1)}(T_{n}^{(1)}))}{\log(1-F^{(2)}(T_{n}^{(2)}))} & =\frac{c^{(1)}(\theta)}{c^{(2)}(\theta)}=\varphi_{1,2}(\theta),
\end{align*}
which is the approximate Bahadur efficiency of $T_{n}^{(1)}$ relative
to $T_{n}^{(2)}$.
\end{thm}
With Theorem \ref{thm:gleser1964_slope}, the difficulty is in showing
that $F\in\mathcal{D}(a,t)$ for some $a>0,t>0$. Typically verification
of the assumption 2 of Theorem \ref{thm:gleser1964_slope} poses no
problem. \cite{Bah1960} showed that the CDF of $\mathcal{N}(0,1)$
belongs to $\mathcal{D}(1,2)$ and the CDF of $\chi_{k}^{2}$ (chi-squared
distribution with $k$ degrees of freedom, fixed $k$) belongs to
$\mathcal{D}(1,1)$. The following results make it easier to determine
whether a given CDF is in the class $\mathcal{D}(a,t)$.
\begin{thm}[{\cite[Theorem 6, 7]{Gle1966}}]
\label{thm:cdf_in_class} Let $X$ have CDF $F\in\mathcal{D}(a,t)$,
and $X_{1},\ldots,X_{m}$ be independent random variables, each with
CDF $F_{i}\in\mathcal{D}(a,t)$. Then, the following statements are
true.
\begin{enumerate}
\item If $b>0$, then the CDF of $bX$ is in $\mathcal{D}(ab^{-t},t)$.
\item $X-b$ has CDF in $\mathcal{D}(a,t)$ provided that $t\ge1$.
\item For $r>0$, $X^{r}$ has CDF in $\mathcal{D}(a,r^{-1}t)$ provided
that $F(0)=0$.
\item $\max(X_{1},\ldots,X_{m})$ has CDF in $\mathcal{D}(a,t)$.
\item Let $a_{1},\ldots,a_{m}$ be non-negative real numbers such that $a_{max}:=\max(a_{1},\ldots,a_{m})>0$.
Then, $\sum_{i=1}^{m}a_{i}X_{i}$ has CDF in $\mathcal{D}(a\cdot a_{max}^{-t},t)$
provided that $\sum_{i=1}^{m}X_{i}$ has CDF in $\mathcal{D}(a,t)$
and $X_{i}\ge0$ for all $i=1,\ldots,m$.
\end{enumerate}
\end{thm}

\section{Proof of Theorem \ref{thm:sigma_q_consistent} }

\label{sec:proof_sigma_q_consistent}Recall Theorem \ref{thm:sigma_q_consistent}:
\sigmaqconsistent*
\begin{proof}
Under $H_{0}$, $p=q$ implies that $\hat{\boldsymbol{\Sigma}}_{q}=\hat{\boldsymbol{\Sigma}}_{p}$
(empirical estimate of $\boldsymbol{\Sigma}_{p}$). Let $\lambda_{j}(A)$
denote the $j^{th}$ eigenvalue of the matrix $A$. Lemma \ref{lem:eigenvalue_cont}
implies that $A\mapsto\lambda_{j}(A)$ is continuous on the space
of real symmetric matrices, for all $j$. Since $\plim_{n\to\infty}\|\hat{\boldsymbol{\Sigma}}_{p}-\boldsymbol{\Sigma}_{p}\|=0$,
by the continuous mapping theorem, the eigenvalues of $\hat{\boldsymbol{\Sigma}}_{p}$
converge to the eigenvalues of $\boldsymbol{\Sigma}_{p}$ in probability.
This implies that $\sum_{i=1}^{dJ}(Z_{i}^{2}-1)\hat{\nu_{i}}$ converges
in probability to $\sum_{i=1}^{dJ}(Z_{i}^{2}-1)\omega_{i}$ as $n\to\infty$,
where $\{\omega_{i}\}_{i=1}^{dJ}$ are eigenvalues of $\boldsymbol{\Sigma}_{p}$.
By Lemma \ref{lem:quantile_converges}, the quantile also converges,
and the test threshold thus matches that of the true asymptotic null
distribution given in claim 1 of Proposition \ref{prop:fssd_asymp_dists}.

Assume $H_{1}$ holds. Let $\hat{t}_{\alpha},t_{\alpha}$ be $(1-\alpha)$-quantiles
of the distributions of $\sum_{i=1}^{dJ}(Z_{i}^{2}-1)\hat{\nu_{i}}$
and $\sum_{i=1}^{dJ}(Z_{i}^{2}-1)\nu_{i}$, respectively, where $\{\nu_{i}\}_{i=1}^{dJ}$
are eigenvalues of $\boldsymbol{\Sigma}_{q}$. By the same argument
as in the previous paragraph, $\hat{t}_{\alpha}$ converges in probability
to $t_{\alpha}$, which is a constant independent of the sample size
$n$. Given $\{\mathbf{v}_{j}\}_{j=1}^{J}\sim\eta$, where $\eta$
is a distribution with a density, $\mathrm{FSSD}^{2}>0$ by Theorem
\ref{thm:fssd}. It follows that 
\begin{align*}
\lim_{n\to\infty}\mathbb{P}\left(n\widehat{\mathrm{FSSD^{2}}}>\hat{t}_{\alpha}\right) & =\lim_{n\to\infty}\mathbb{P}\left(\widehat{\mathrm{FSSD^{2}}}-\frac{\hat{t}_{\alpha}}{n}>0\right)\stackrel{(a)}{=}\mathbb{P}\left(\mathrm{FSSD^{2}}>0\right)=1,
\end{align*}
where at $(a)$, we use the fact that $\widehat{\mathrm{FSSD^{2}}}$
converges in probability to $\mathrm{FSSD^{2}}$ by the law of large
numbers, and that $\lim_{n\to\infty}\hat{t}_{\alpha}/n=0$.
\end{proof}

\section{Proof of Theorem \ref{thm:fssd_slope} (Slope of $n\widehat{\mathrm{FSSD^{2}}}$)}

\label{sec:proof_fssdslope}Recall Theorem \ref{thm:fssd_slope}:
\fssdslope*
\begin{proof}
We will use Theorem \ref{thm:gleser1964_slope} to derive the slope.
For the assumption 1 of Theorem \ref{thm:gleser1964_slope}, we first
show that the asymptotic null distribution belongs to the class $\mathcal{D}(a=1/\omega_{1},t=1)$
as defined in Definition \ref{def:nulldist_class}. By Proposition
\ref{prop:fssd_asymp_dists}, the asymptotic null distribution is
$\sum_{i=1}^{dJ}\omega_{i}Z_{i}^{2}-\sum_{i=1}^{dJ}\omega_{i}$ where
$Z_{1},\ldots,Z_{dJ}\stackrel{i.i.d.}{\sim}\mathcal{N}(0,1)$ and
$\omega_{1}\ge\cdots\ge\omega_{dJ}\ge0$ are eigenvalues of $\boldsymbol{\Sigma}_{p}$.
It is known from \cite{Bah1960} that the CDF of $\chi_{f}^{2}$ is
in $\mathcal{D}(1,1)$ for any fixed degrees of freedom $f$. Thus,
it follows from claim 5 of Theorem \ref{thm:cdf_in_class} that the
CDF of $\sum_{i=1}^{dJ}\omega_{i}Z_{i}^{2}$ is in $\mathcal{D}(a=1/\omega_{1},t=1)$.
Claim 2 of Theorem \ref{thm:cdf_in_class} guarantees that the CDF
of $\sum_{i=1}^{dJ}\omega_{i}Z_{i}^{2}-\sum_{i=1}^{dJ}\omega_{i}$
is in $\mathcal{D}(a=1/\omega_{1},t=1)$ as desired.

For assumption 2 of Theorem \ref{thm:gleser1964_slope}, choose $R(n):=n$.
It follows from the weak law of large numbers that under $H_{1}$,
$n\widehat{\mathrm{FSSD^{2}}}/R(n)\stackrel{p}{\to}\mathrm{FSSD^{2}}$.
By Theorem \ref{thm:gleser1964_slope}, the approximate slope is $\mathrm{FSSD^{2}}/\omega_{1}$. 
\end{proof}

\section{Proof of Theorem \ref{thm:lks_slope} (Slope of $\sqrt{n}\widehat{S_{l}^{2}}$) }

Recall Theorem \ref{thm:lks_slope}: \lksslope*
\begin{proof}
We will use Theorem \ref{thm:gleser1964_slope} to derive the slope.
By the central limit theorem,
\begin{align*}
\sqrt{n}\left(\widehat{S_{l}^{2}}-S_{p}^{2}(q)\right) & \stackrel{d}{\to}\mathcal{N}(0,2\mathbb{V}_{q}[h_{p}(\mathbf{x},\mathbf{x}')]),
\end{align*}
where $\mathbb{V}_{q}[h_{p}(\mathbf{x},\mathbf{x}')]:=\mathbb{E}_{\mathbf{x}\sim q}\mathbb{E}_{\mathbf{x}'\sim q}[h_{p}^{2}(\mathbf{x},\mathbf{x}')]-\left(\mathbb{E}_{\mathbf{x}\sim q}\mathbb{E}_{\mathbf{x}'\sim q}[h_{p}(\mathbf{x},\mathbf{x}')]\right)^{2}$.
Under $H_{0}:p=q$, it follows that $S_{p}^{2}(q)=\mathbb{E}_{\mathbf{x}\sim q}\mathbb{E}_{\mathbf{x}'\sim q}[h_{p}(\mathbf{x},\mathbf{x}')]=0$
by Theorem \ref{thm:stein_iff_pq}, and $\sqrt{n}\widehat{S_{l}^{2}}\stackrel{d}{\to}\mathcal{N}(0,2\mathbb{V}_{p}[h_{p}(\mathbf{x},\mathbf{x}')])$
where $\mathbb{V}_{p}[h_{p}(\mathbf{x},\mathbf{x}')]:=\mathbb{E}_{\mathbf{x}\sim p}\mathbb{E}_{\mathbf{x}'\sim p}[h_{p}^{2}(\mathbf{x},\mathbf{x}')]$.
It is known from \cite{Bah1960} that the CDF of $\mathcal{N}(0,1)$
is in the class $\mathcal{D}(1,2)$ (see Definition \ref{def:nulldist_class}).
Thus, by property 1 of Theorem \ref{thm:cdf_in_class}, the CDF of
$\mathcal{N}(0,2\mathbb{V}_{p}[h_{p}(\mathbf{x},\mathbf{x}')])$ is
in $\mathcal{D}\left(a=\frac{1}{2\mathbb{V}_{p}[h_{p}(\mathbf{x},\mathbf{x}')]},t=2\right)$.

For assumption 2 of Theorem \ref{thm:gleser1964_slope}, choose $R(n):=\sqrt{n}$.
It follows from the weak law of large numbers that under $H_{1}$,
$\sqrt{n}\widehat{S_{l}^{2}}/R(n)=\widehat{S_{l}^{2}}\stackrel{p}{\to}S_{p}^{2}(q)$.
By Theorem \ref{thm:gleser1964_slope}, the approximate slope is $\frac{S_{p}^{4}(q)}{2\mathbb{V}_{p}[h_{p}(\mathbf{x},\mathbf{x}')]}$. 
\end{proof}

\section{Proof of Theorem \ref{thm:effgaussmean}}

\label{sec:proof_effgaussmean}We will first prove a number of useful
results that will allow us to prove Theorem \ref{thm:effgaussmean}
at the end. Recall that $v$ denotes a test location in the FSSD test,
$\sigma_{k}^{2}$ denotes the Gaussian kernel bandwidth of the FSSD
test, and $\kappa^{2}$ denotes the Gaussian kernel bandwidth of the
LKS test.
\begin{prop}
\label{prop:fssd_slope_gauss}Under the assumption that $J=1$ (i.e.,
one test location $v$), $p=\mathcal{N}(0,1)$ and $q=\mathcal{N}(\mu_{q},\sigma_{q}^{2})$,
the approximate Bahadur Slope of $n\widehat{\mathrm{FSSD^{2}}}$ is
\begin{equation}
c^{(\mathrm{FSSD)}}:=\frac{\left(\sigma_{k}^{2}\right){}^{3/2}\left(\sigma_{k}^{2}+2\right){}^{5/2}e^{\frac{v^{2}}{\sigma_{k}^{2}+2}-\frac{\left(v-\mu_{q}\right){}^{2}}{\sigma_{k}^{2}+\sigma_{q}^{2}}}\left(\left(\sigma_{k}^{2}+1\right)\mu_{q}+v\left(\sigma_{q}^{2}-1\right)\right)^{2}}{\left(\sigma_{k}^{2}+\sigma_{q}^{2}\right){}^{3}\left(\sigma_{k}^{6}+4\sigma_{k}^{4}+\left(v^{2}+5\right)\sigma_{k}^{2}+2\right)}.\label{eq:fssd_slope_gauss}
\end{equation}
\end{prop}
\begin{proof}
This result follows directly from Theorem \ref{thm:fssd_slope} specialized
to the case of $p=\mathcal{N}(0,1)$, $q=\mathcal{N}(\mu_{q},\sigma_{q}^{2})$,
and $J=1$. Since $dJ=1$, the covariance matrix 
\[
\boldsymbol{\Sigma}_{p}=\mathbb{E}_{x\sim p}\left[\xi_{p}^{2}(x,v)\right]=\frac{e^{-\frac{v^{2}}{\sigma_{k}^{2}+2}}\left(\sigma_{k}^{6}+4\sigma_{k}^{4}+\left(v^{2}+5\right)\sigma_{k}^{2}+2\right)}{\sigma_{k}\left(\sigma_{k}^{2}+2\right){}^{5/2}}
\]
 reduces to a scalar, where $\xi_{p}(x,v)=\left[\frac{\partial}{\partial x}\log p(x)\right]k(x,v)+\frac{\partial}{\partial x}k(x,v)=-e^{-\frac{(v-x)^{2}}{2\sigma_{k}^{2}}}\left(x\sigma_{k}^{2}-v+x\right)/\sigma_{k}^{2}$.
In this case, 
\[
\mathrm{FSSD^{2}}=\mathbb{E}_{x\sim q}^{2}\left[\xi_{p}(x,v)\right]=\frac{\sigma_{k}^{2}e^{-\frac{\left(v-\mu_{q}\right){}^{2}}{\sigma_{k}^{2}+\sigma_{q}^{2}}}\left(\left(\sigma_{k}^{2}+1\right)\mu_{q}+v\left(\sigma_{q}^{2}-1\right)\right)^{2}}{\left(\sigma_{k}^{2}+\sigma_{q}^{2}\right)^{3}}.
\]
Taking the ratio $\mathrm{FSSD^{2}}/\mathbb{E}_{x\sim p}\left[\xi_{p}^{2}(x,v)\right]$
gives the result.
\end{proof}
\begin{prop}
\label{prop:lks_slope_gauss}Assume that $p=\mathcal{N}(0,1)$ and
$q=\mathcal{N}(\mu_{q},\sigma_{q}^{2})$. Let $\sqrt{n}\widehat{S_{l}^{2}}$
be the linear-time kernel Stein (LKS) test statistic where $\widehat{S_{l}^{2}}$
is defined in Section \ref{sec:kstein_test} with a Gaussian kernel
$k(x,y)=\exp\left(-\frac{(x-y)^{2}}{2\kappa^{2}}\right)$. Then, the
following statements hold.
\begin{enumerate}
\item The population kernel Stein discrepancy is
\begin{align*}
S_{p}^{2}(q) & =\frac{\mu_{q}^{2}\left(\kappa^{2}+2\sigma_{q}^{2}\right)+\left(\sigma_{q}^{2}-1\right){}^{2}}{\left(\kappa^{2}+2\sigma_{q}^{2}\right)\sqrt{\frac{2\sigma_{q}^{2}}{\kappa^{2}}+1}}.
\end{align*}
\item The approximate Bahadur slope of $\sqrt{n}\widehat{S_{l}^{2}}$ is
\begin{equation}
c^{(\mathrm{LKS})}:=\frac{\kappa^{5}\left(\kappa^{2}+4\right)^{5/2}\left[\mu_{q}^{2}\left(\kappa^{2}+2\sigma_{q}^{2}\right)+\left(\sigma_{q}^{2}-1\right)^{2}\right]^{2}}{2\left(\kappa^{8}+8\kappa^{6}+21\kappa^{4}+20\kappa^{2}+12\right)\left(\kappa^{2}+2\sigma_{q}^{2}\right)^{3}}.\label{eq:lks_slope_gauss}
\end{equation}
\item Let 
\[
c_{1}^{(\mathrm{LKS})}=\frac{\left(\kappa^{2}\right)^{5/2}\left(\kappa^{2}+4\right)^{5/2}\mu_{q}^{4}}{2\left(\kappa^{2}+2\right)\left(\kappa^{8}+8\kappa^{6}+21\kappa^{4}+20\kappa^{2}+12\right)}
\]
 denote the approximate slope $c^{(\mathrm{LKS})}$ specialized to
when $q=\mathcal{N}(\mu_{q},1)$. Then, for any $\mu_{q}\neq0$, the
function $\kappa^{2}\mapsto c_{1}^{(\mathrm{LKS})}(\mu_{q},\kappa^{2})$
is strictly increasing on $(0,\infty)$. Further, 
\begin{equation}
\lim_{\kappa^{2}\to\infty}c_{1}^{(\mathrm{LKS})}(\mu_{q},\kappa^{2})=\mu_{q}^{4}/2.\label{eq:lks_slope_vq1_bounds}
\end{equation}
\end{enumerate}
\end{prop}
\begin{proof}
\textbf{Proof of Claim 1, 2}. Recall $\widehat{S_{l}^{2}}:=\frac{2}{n}\sum_{i=1}^{n/2}h_{p}(x_{2i-1},x_{2i})$.
With $p=\mathcal{N}(0,1)$, and $k(x,y)=\exp\left(-\frac{(x-y)^{2}}{2\kappa^{2}}\right)$,
$h_{p}(x,y)$ can be written as
\begin{align*}
h_{p}(x,y) & :=\frac{e^{-\frac{(x-y)^{2}}{2\kappa^{2}}}\left(\kappa^{2}-\left(\kappa^{2}+1\right)x^{2}+\left(\kappa^{4}+2\kappa^{2}+2\right)xy-\left(\kappa^{2}+1\right)y^{2}\right)}{\kappa^{4}}.
\end{align*}
By Theorem \ref{thm:lks_slope}, $c^{(\mathrm{LKS})}=\frac{1}{2}\frac{\left[\mathbb{E}_{q}h_{p}(\mathbf{x},\mathbf{x}')\right]^{2}}{\mathbb{E}_{p}\left[h_{p}^{2}(\mathbf{x},\mathbf{x}')\right]}$
which mainly involves expectations with respect to a normal distribution.
In computing the expectation $\mathbb{E}_{x'\sim q}h_{p}(x,x')$,
the idea is to form the density for a new normal distribution by combining
$\frac{1}{\sqrt{2\pi\sigma_{q}^{2}}}e^{-(x-\mu_{q})^{2}/2\sigma_{q}^{2}}$
(the density of $q$) and the term $e^{-\frac{(x-y)^{2}}{2\kappa^{2}}}$
in the expression of $h_{p}(x,y)$. Computation of $\mathbb{E}_{x'\sim q}h_{p}(x,x')$
will then boil down to computing an expectation wrt. a new normal
distribution.

It turns out that 
\begin{align*}
\mathbb{E}_{x\sim q}\mathbb{E}_{x'\sim q}[h_{p}(x,x')] & =\frac{\mu_{q}^{2}\left(\kappa^{2}+2\sigma_{q}^{2}\right)+\left(\sigma_{q}^{2}-1\right)^{2}}{\left(\kappa^{2}+2\sigma_{q}^{2}\right)\sqrt{\frac{2\sigma_{q}^{2}}{\kappa^{2}}+1}}=S_{p}^{2}(q),\\
\mathbb{E}_{p}\left[h_{p}^{2}(\mathbf{x},\mathbf{x}')\right] & =\frac{\left(\kappa^{2}+4\right)\left(\kappa^{4}+4\kappa^{2}+5\right)\kappa^{2}+12}{\kappa^{3}\left(\kappa^{2}+4\right)^{5/2}}.
\end{align*}
Computing $\frac{1}{2}\frac{S_{p}^{4}(q)}{\mathbb{E}_{p}\left[h_{p}^{2}(\mathbf{x},\mathbf{x}')\right]}$
gives the slope.

\textbf{Proof of Claim 3}. The expression for $c_{1}^{(\mathrm{LKS})}$
is obtained straightforwardly by plugging $\sigma_{q}^{2}=1$ into
the expression of $c^{(\mathrm{LKS})}$. Assume $\mu_{q}\neq0$. It
can be seen that $c_{1}^{(\mathrm{LKS})}(\mu_{q},\kappa^{2})$ is
differentiable with respect to $\kappa^{2}$ on the interval $(0,\infty)$.
The partial derivative is given by
\begin{align*}
\frac{\partial}{\partial\kappa^{2}}c_{1}^{(\mathrm{LKS})} & =\frac{\left(\kappa^{2}\right)^{3/2}\left(\kappa^{2}+4\right)^{3/2}\left(7\kappa^{8}+56\kappa^{6}+166\kappa^{4}+216\kappa^{2}+120\right)\mu_{q}^{4}}{\left(\kappa^{2}+2\right)^{2}\left(\kappa^{8}+8\kappa^{6}+21\kappa^{4}+20\kappa^{2}+12\right)^{2}}.
\end{align*}
Since for any $\mu_{q}\neq0$, $\frac{\partial}{\partial\kappa^{2}}c_{1}^{(\mathrm{LKS})}>0$
for $\kappa^{2}\in(0,\infty)$, we conclude that $\kappa^{2}\mapsto c_{1}^{(\mathrm{LKS})}(\mu_{q},\kappa^{2})$
is a strictly increasing function on $(0,\infty)$. By taking the
limit, we have $\lim_{\kappa^{2}\to\infty}c_{1}^{(\mathrm{LKS})}(\mu_{q},\kappa^{2})=\mu_{q}^{4}/2$. 
\end{proof}
We are ready to prove Theorem \ref{thm:effgaussmean}. Recall that
$\sigma_{k}^{2}$ is the kernel bandwidth of $n\widehat{\mathrm{FSSD^{2}}}$,
and $\kappa^{2}$ is the kernel bandwidth of $\sqrt{n}\widehat{S_{l}^{2}}$
(see Section \ref{sec:kstein_test}). Recall Theorem \ref{thm:effgaussmean}:
\effgaussmean* 
\begin{proof}
By Proposition \ref{prop:fssd_slope_gauss}, the approximate slope
of $n\widehat{\mathrm{FSSD^{2}}}$ when $\sigma_{q}^{2}=1$ is 
\begin{align*}
c_{1}^{(\mathrm{FSSD)}}(\mu_{q},v,\sigma_{k}^{2}) & =\frac{\sigma_{k}^{2}\left(\sigma_{k}^{2}+2\right){}^{3}\mu_{q}^{2}e^{\frac{v^{2}}{\sigma_{k}^{2}+2}-\frac{\left(v-\mu_{q}\right){}^{2}}{\sigma_{k}^{2}+1}}}{\sqrt{\frac{2}{\sigma_{k}^{2}}+1}\left(\sigma_{k}^{2}+1\right)\left(\sigma_{k}^{6}+4\sigma_{k}^{4}+\left(v^{2}+5\right)\sigma_{k}^{2}+2\right)}.
\end{align*}
Theorem \ref{thm:gleser1964_efficiency} states that the approximate
efficiency $E_{1}(\mu_{q},v,\sigma_{k}^{2},\kappa^{2})$ is given
by the ratio $\frac{c_{1}^{(\mathrm{FSSD)}}(\mu_{q},v,\sigma_{k}^{2})}{c_{1}^{(\mathrm{LKS})}(\mu_{q},\kappa^{2})}$
(see Propositions \ref{prop:fssd_slope_gauss} and \ref{prop:lks_slope_gauss})
of the approximate slopes of the two tests. Pick $\sigma_{k}^{2}=1$,
and for any $\mu_{q}\neq0$, pick $v=2\mu_{q}$. These choices give
the slope
\begin{align*}
c_{1}^{(\mathrm{FSSD)}}(\mu_{q},2\mu_{q},1) & =\frac{9\sqrt{3}e^{\frac{5\mu_{q}^{2}}{6}}\mu_{q}^{2}}{2\left(4\mu_{q}^{2}+12\right)}.
\end{align*}
We have
\begin{align*}
E_{1}(\mu_{q},v,\sigma_{k}^{2},\kappa^{2}) & =E_{1}(\mu_{q},2\mu_{q},1,\kappa^{2})\\
 & =c_{1}^{(\mathrm{FSSD)}}(\mu_{q},2\mu_{q},1)/c_{1}^{(\mathrm{LKS})}(\mu_{q},\kappa^{2})\\
 & \stackrel{(a)}{\ge}c_{1}^{(\mathrm{FSSD)}}(\mu_{q},2\mu_{q},1)/\left(\frac{\mu_{q}^{4}}{2}\right)\\
 & =\frac{9\sqrt{3}e^{\frac{5\mu_{q}^{2}}{6}}}{\mu_{q}^{2}\left(4\mu_{q}^{2}+12\right)}:=g(\mu_{q}),
\end{align*}
where at $(a)$ we use $c_{1}^{(\mathrm{LKS})}(\mu_{q},\kappa^{2})\le\mu_{q}^{4}/2$
from (\ref{eq:lks_slope_vq1_bounds}). It can be seen that for $\mu_{q}\neq0$,
$g(\mu_{q})$ is an even function i.e., $g(\mu_{q})=g(-\mu_{q})$.
The second derivative 
\begin{align*}
\frac{\partial^{2}}{\partial\mu_{q}^{2}}g(\mu_{q}) & =\sqrt{3}e^{\frac{5\mu_{q}^{2}}{6}}\left(25\mu_{q}^{8}+45\mu_{q}^{6}-45\mu_{q}^{4}+81\mu_{q}^{2}+486\right)/\left(4\mu_{q}^{4}\left(\mu_{q}^{2}+3\right)^{3}\right)>0.
\end{align*}
To see that $\frac{\partial^{2}}{\partial\mu_{q}^{2}}g(\mu_{q})>0$,
consider two cases of $\mu_{q}^{2}\ge1$ and $0<\mu_{q}^{2}<1$. When
$\mu_{q}^{2}\ge1$,
\begin{align*}
g(\mu_{q}) & \ge\sqrt{3}e^{\frac{5\mu_{q}^{2}}{6}}\left(25\mu_{q}^{8}+81\mu_{q}^{2}+486\right)/\left(4\mu_{q}^{4}\left(\mu_{q}^{2}+3\right)^{3}\right)>0,
\end{align*}
because $45\mu_{q}^{6}-45\mu_{q}^{4}\ge0$. When $0<\mu_{q}^{2}<1$,
\begin{align*}
g(\mu_{q}) & \ge\sqrt{3}e^{\frac{5\mu_{q}^{2}}{6}}\left(25\mu_{q}^{8}+45\mu_{q}^{6}+486\right)/\left(4\mu_{q}^{4}\left(\mu_{q}^{2}+3\right)^{3}\right)>0,
\end{align*}
because $-45\mu_{q}^{4}+81\mu_{q}^{2}\ge0$. This shows that $g(\mu_{q})$
is convex on $(0,\infty)$. The function $g(\mu_{q})$ on $\mathbb{R}\backslash\{0\}$
achieves global minima at $\mu_{q}=\mu_{q}^{*}:=\pm\sqrt{\frac{3}{10}\left(\sqrt{41}-1\right)}\approx\pm1.273$.
This implies that 
\begin{align*}
E_{1}(\mu_{q},v,\sigma_{k}^{2},\kappa^{2}) & \ge g(\mu_{q})\ge g(\mu_{q}^{*})\\
 & =\frac{25\sqrt{3}e^{\frac{1}{4}\left(\sqrt{41}-1\right)}}{8\left(\sqrt{41}+4\right)}\approx2.00855>2.
\end{align*}
\end{proof}

\section{Known Results}

This section presents known results from other works.
\begin{thm}[{\cite[Theorem 2.2]{Chwialkowski2016}}]
\label{thm:stein_iff_pq} If the kernel $k$ is $C_{0}$-universal
\cite[Definition 4.1]{Carmeli2010}, $\mathbb{E}_{\mathbf{x}\sim q}\mathbb{E}_{\mathbf{x}'\sim q}h_{p}(\mathbf{x},\mathbf{x}')<\infty$,
and $\mathbb{E}_{\mathbf{x}\sim q}\|\nabla_{\mathbf{x}}\log\frac{p(\mathbf{x})}{q(\mathbf{x})}\|^{2}<\infty$,
then $S_{p}(q)=\|\mathbb{E}_{\mathbf{x}\sim q}\xi_{p}(\mathbf{x},\cdot)\|_{\mathcal{F}^{d}}=0$
if and only if $p=q$. 
\end{thm}
\begin{lem}[{\cite[Lemma 1]{ChwRamSejGre15}}]
\label{lem:rkhs_analytic} Let $U$ be an open subset of $\mathbb{R}^{d}$.
If $k$ is a bounded, analytic kernel on $U\times U$, then all functions
in the RKHS associated with $k$ are analytic.\footnote{The result of \cite{ChwRamSejGre15} considers only the case where
$U=\mathbb{R}^{d}$. However, the same proof goes through for any
open subset $U\subseteq\mathbb{R}^{d}$.}
\end{lem}
\begin{lem}[{Weyl's Perturbation Theorem \cite[p.\ 152]{Bha2013}}]
\label{lem:eigenvalue_cont} Let $\lambda_{j}(A)$ denote the $j^{th}$
eigenvalue of a square matrix $A$. If $A,B$ are two Hermitian matrices,
then 
\[
\max_{j}|\lambda_{j}(A)-\lambda_{j}(B)|\le\|A-B\|,
\]
where $\|\cdot\|$ denotes the operator norm.
\end{lem}
\begin{lem}[{\cite[Lemma 21.2]{Vaa2000}}]
\label{lem:quantile_converges} For any sequence of cumulative distribution
functions, $F_{n}^{-1}\stackrel{d}{\to}F^{-1}$ if and only if $F_{n}\stackrel{d}{\to}F$. 
\end{lem}

\bibliographysup{kgof_appendix}
\end{document}